\def\year{2022}\relax
\documentclass[letterpaper]{article} % DO NOT CHANGE THIS
\usepackage{aaai22}  % DO NOT CHANGE THIS
\usepackage{times}  % DO NOT CHANGE THIS
\usepackage{helvet}  % DO NOT CHANGE THIS
\usepackage{courier}  % DO NOT CHANGE THIS
\usepackage[hyphens]{url}  % DO NOT CHANGE THIS
\usepackage{graphicx} % DO NOT CHANGE THIS
\usepackage{natbib}  % DO NOT CHANGE THIS AND DO NOT ADD ANY OPTIONS TO IT
\usepackage{caption} % DO NOT CHANGE THIS AND DO NOT ADD ANY OPTIONS TO IT
\DeclareCaptionStyle{ruled}{labelfont=normalfont,labelsep=colon,strut=off} % DO NOT CHANGE THIS
\usepackage{soul}

\usepackage[utf8]{inputenc} % allow utf-8 input
\usepackage[T1]{fontenc}    % use 8-bit T1 fonts
\usepackage{hyperref}       % hyperlinks
\usepackage{url}            % simple URL typesetting
\usepackage{booktabs}       % professional-quality tables
\usepackage{amsfonts}       % blackboard math symbols
\usepackage{nicefrac}       % compact symbols for 1/2, etc.
\usepackage{microtype}      % microtypography
\usepackage{xcolor}         % colors

\usepackage{microtype}
\usepackage{graphicx}
\usepackage{subfigure}
\usepackage{booktabs} % for professional tables

\usepackage{paralist}
\usepackage{todonotes}
\usepackage{bm}
\usepackage{siunitx}

\usepackage{hyperref}

\usepackage{amsmath}
\usepackage{amssymb}
\usepackage{amsfonts}
\usepackage{amsthm}
\usepackage{mathrsfs, euscript}
\usepackage{color}
\usepackage{xcolor}
\usepackage{nccmath}
\usepackage{wrapfig}
\usepackage[ruled,vlined,linesnumbered ]{algorithm2e}
\DeclareMathOperator*{\argmax}{arg\,max}
\DeclareMathOperator*{\argmin}{arg\,min}

\title{Context Uncertainty in Contextual Bandits\\with Applications to Recommender Systems}

\author{
    Hao Wang\textsuperscript{\rm 1,2}\thanks{Work done while at AWS AI Labs.},
    Yifei Ma\textsuperscript{\rm 1},
    Hao Ding\textsuperscript{\rm 1},
    Yuyang Wang\textsuperscript{\rm 1}
}
\affiliations{
    \textsuperscript{\rm 1}AWS AI Labs \textsuperscript{\rm 2}Department of Computer Science, Rutgers University\\    
    hw488@cs.rutgers.edu, \{yifeim,haodin,yuyawang\}@amazon.com
}

\usepackage{bibentry}
\begin{document}

\maketitle

\begin{abstract}
Recurrent neural networks have proven effective in modeling sequential user feedbacks for recommender systems. However, they usually focus solely on item relevance and fail to effectively explore diverse items for users, therefore harming the system performance in the long run. To address this problem, we propose a new type of recurrent neural networks, dubbed recurrent exploration networks (REN), to jointly perform representation learning and effective exploration in the latent space. REN tries to balance relevance and exploration while taking into account the uncertainty in the representations. Our theoretical analysis shows that REN can preserve the rate-optimal sublinear regret even when there exists uncertainty in the learned representations. Our empirical study demonstrates that REN can achieve satisfactory long-term rewards on both synthetic and real-world recommendation datasets, outperforming state-of-the-art models. 
\end{abstract}

\def\Blue{\color{blue}}
\def\Purple{\color{purple}}

\def\A{{\bf A}}
\def\a{{\bf a}}
\def\B{{\bf B}}
\def\b{{\bf b}}
\def\C{{\bf C}}
\def\c{{\bf c}}
\def\D{{\bf D}}
\def\d{{\bf d}}
\def\E{{\bf E}}
\def\e{{\bf e}}
\def\f{{\bf f}}
\def\F{{\bf F}}
\def\K{{\bf K}}
\def\k{{\bf k}}
\def\L{{\bf L}}
\def\H{{\bf H}}
\def\h{{\bf h}}
\def\G{{\bf G}}
\def\g{{\bf g}}
\def\I{{\bf I}}
\def\J{{\bf J}}
\def\R{{\bf R}}
\def\X{{\bf X}}
\def\Y{{\bf Y}}
\def\OO{{\bf O}}
\def\oo{{\bf o}}
\def\P{{\bf P}}
\def\p{{\bf p}}
\def\Q{{\bf Q}}
\def\q{{\bf q}}
\def\r{{\bf r}}
\def\s{{\bf s}}
\def\S{{\bf S}}
\def\t{{\bf t}}
\def\T{{\bf T}}
\def\x{{\bf x}}
\def\y{{\bf y}}
\def\z{{\bf z}}
\def\Z{{\bf Z}}
\def\M{{\bf M}}
\def\m{{\bf m}}
\def\n{{\bf n}}
\def\U{{\bf U}}
\def\u{{\bf u}}
\def\V{{\bf V}}
\def\v{{\bf v}}
\def\W{{\bf W}}
\def\w{{\bf w}}
\def\0{{\bf 0}}
\def\1{{\bf 1}}

\def\AM{{\mathcal A}}
\def\EM{{\mathcal E}}
\def\FM{{\mathcal F}}
\def\TM{{\mathcal T}}
\def\UM{{\mathcal U}}
\def\XM{{\mathcal X}}
\def\YM{{\mathcal Y}}
\def\NM{{\mathcal N}}
\def\OM{{\mathcal O}}
\def\IM{{\mathcal I}}
\def\GM{{\mathcal G}}
\def\PM{{\mathcal P}}
\def\LM{{\mathcal L}}
\def\MM{{\mathcal M}}
\def\DM{{\mathcal D}}
\def\SM{{\mathcal S}}
\def\ZM{{\mathcal Z}}
\def\RB{{\mathbb R}}
\def\EB{{\mathbb E}}
\def\VB{{\mathbb V}}

\def\tx{\tilde{\bf x}}
\def\ty{\tilde{\bf y}}
\def\tz{\tilde{\bf z}}
\def\hd{\hat{d}}
\def\HD{\hat{\bf D}}
\def\hx{\hat{\bf x}}
\def\hR{\hat{R}}

\def\Ome{\mbox{\boldmath$\omega$\unboldmath}}
\def\Om{\mbox{\boldmath$\Omega$\unboldmath}}
\def\bet{\mbox{\boldmath$\beta$\unboldmath}}
\def\et{\mbox{\boldmath$\eta$\unboldmath}}
\def\ep{\mbox{\boldmath$\epsilon$\unboldmath}}
\def\ph{\mbox{\boldmath$\phi$\unboldmath}}
\def\Pii{\mbox{\boldmath$\Pi$\unboldmath}}
\def\pii{\mbox{\boldmath$\pi$\unboldmath}}
\def\Ph{\mbox{\boldmath$\Phi$\unboldmath}}
\def\Ps{\mbox{\boldmath$\Psi$\unboldmath}}
\def\tha{\mbox{\boldmath$\theta$\unboldmath}}
\def\Tha{\mbox{\boldmath$\Theta$\unboldmath}}
\def\muu{\mbox{\boldmath$\mu$\unboldmath}}
\def\Muu{\mbox{\boldmath$\mathcal{M}$\unboldmath}}
\def\Si{\mbox{\boldmath$\Sigma$\unboldmath}}
\def\si{\mbox{\boldmath$\sigma$\unboldmath}}
\def\Gam{\mbox{\boldmath$\Gamma$\unboldmath}}
\def\gamm{\mbox{\boldmath$\gamma$\unboldmath}}
\def\Lam{\mbox{\boldmath$\Lambda$\unboldmath}}
\def\De{\mbox{\boldmath$\Delta$\unboldmath}}
\def\vps{\mbox{\boldmath$\varepsilon$\unboldmath}}
\def\Up{\mbox{\boldmath$\Upsilon$\unboldmath}}
\def\xii{\mbox{\boldmath$\xi$\unboldmath}}
\def\Xii{\mbox{\boldmath$\Xi$\unboldmath}}
\def\Lap{\mbox{\boldmath$\LM$\unboldmath}}
\newcommand{\ti}[1]{\tilde{#1}}

\def\tr{\mathrm{tr}}
\def\etr{\mathrm{etr}}
\def\etal{{\em et al.\/}\,}
\newcommand{\indep}{{\;\bot\!\!\!\!\!\!\bot\;}}
\def\argmax{\mathop{\rm argmax}}
\def\argmin{\mathop{\rm argmin}}
\def\vec{\text{vec}}
\def\cov{\text{cov}}
\def\dg{\text{diag}}

\newtheorem{observation}{\textbf{Observation}}
\newtheorem{remark}{Remark}
\newtheorem{theorem}{Theorem}
\newtheorem{lemma}{Lemma}
\newtheorem{definition}{Definition}
\newtheorem{problem}{Problem}
\newtheorem{proposition}{Proposition}
\newtheorem{cor}{Corollary}
\newtheorem{assup}{Assumption}
% \numberwithin{theorem}{section}
% \numberwithin{lemma}{section}
% \numberwithin{remark}{section}
% \numberwithin{cor}{section}
% \numberwithin{assup}{section}
% \numberwithin{alg}{section}

\newcommand{\tabref}[1]{Table~\ref{#1}}
\newcommand{\secref}[1]{Sec.~\ref{#1}}
\newcommand{\figref}[1]{Fig.~\ref{#1}}
\newcommand{\lemref}[1]{Lemma~\ref{#1}}
\newcommand{\thmref}[1]{Theorem~\ref{#1}}
\newcommand{\clmref}[1]{Claim~\ref{#1}}
\newcommand{\crlref}[1]{Corollary~\ref{#1}}
\newcommand{\eqnref}[1]{Eqn.~\ref{#1}}
\newcommand{\algref}[1]{Algorithm~\ref{#1}}
\newcommand{\bernie}[1]{\textcolor{blue}{BW: #1}}
\newcommand{\hw}[1]{\textcolor{red}{HW:#1}}
\newcommand{\bx} {{\bm x}}

\section{Introduction}
\label{sec:intro}
Modeling and predicting sequential user feedbacks is a core problem in modern e-commerce recommender systems. In this regard, recurrent neural networks (RNN) have shown great promise since they can naturally handle sequential data~\citep{GRU4Rec,HGRU,belletti2019quantifying,HRNN}. While these RNN-based models can effectively learn representations in the latent space to achieve satisfactory immediate recommendation accuracy, they typically focus solely on relevance and fall short of effective exploration in the latent space, leading to poor performance in the long run. For example, a recommender system may keep recommending action movies to a user once it learns that she likes such movies. This may increase immediate rewards, but the lack of exploration in other movie genres can certainly be detrimental to long-term rewards. % TODO: consider a better example

So, how does one effectively explore diverse items for users while retaining the representation power offered by RNN-based recommenders. We note that the learned representations in the latent space are crucial for these models' success. Therefore we propose recurrent exploration networks (REN) to explore diverse items in the latent space learned by RNN-based models. REN tries to balance relevance and exploration during recommendations using the learned representations. 

One roadblock is that effective exploration relies heavily on well learned representations, which in turn require sufficient exploration; this is a chicken-and-egg problem. In a case where RNN learns unreasonable representations (e.g., all items have the same representations), exploration in the latent space is meaningless. To address this problem, we enable REN to take into account the uncertainty of the learned representations as well during recommendations. Essentially items whose representations have higher uncertainty can be explored more often. Such a model can be seen as a contextual bandit algorithm that is aware of the uncertainty for each context. Our contributions are as follows:
\begin{enumerate}
% \begin{compactenum}
    \item We propose REN as a new type of RNN to balance relevance and exploration during recommendation, yielding satisfactory long-term rewards.
    \item Our theoretical analysis shows that there is an upper confidence bound related to uncertainty in learned representations. With such a bound implemented in the algorithm, REN can achieve the same rate-optimal sublinear regret. To the best of our knowledge, we are the first to study the regret bounds under ``context uncertainty''. % as~\citet{chu2011contextual}. %To the best of our knowledge, we are the first to study the regret bounds under ``context" uncertainty. 
    \item Experiments of joint learning and exploration on both synthetic and real-world temporal datasets show that REN significantly improve long-term rewards over state-of-the-art RNN-based recommenders.
% \end{compactenum}
\end{enumerate}
% The rest of the paper is organized as follows. We discuss the related work in Section 2, followed by introducing our model in the subsequent section. Section 4 is devoted to theoretical analysis. We show the effectiveness of our approach in Section 5 and then make closing remark. % in the last section.
% add back this summary if there is enough space

% Experiments on both synthetic and real-world datasets show that REN significantly improve long-term rewards over state-of-the-art RNN-based recommenders.
% \begin{enumerate}
%     \item We propose a loss function that balances between relevance and diversity, which can be easily integrates into any sequence model, such as RNN, Transformer, Wavenet, etc.
%     \item We establish the connections between the problem of optimizing for diversity and the maximization for long-term rewards. Furthermore, we show that the former leads to the latter, under mild assumptions. 
% \end{enumerate}

\section{Related Work}\label{sec:related}
\textbf{Deep Learning for Recommender Systems.}
Deep learning (DL) has been playing a key role in modern recommender systems~\citep{RBM4CF,deepmusic,CDL,RSDAE,CRAE,CVAE,chen2019top,fang2019deep,tang2019towards,ZESRec,ExposureBias}. \cite{RBM4CF} uses restricted Boltzmann machine to perform collaborative filtering in recommender systems. Collaborative deep learning (CDL)~\cite{CDL,CRAE,CVAE} is devised as Bayesian deep learning models~\cite{BDL,BDLSurvey,BDLThesis} to significantly improve recommendation performance. In terms of sequential (or session-based) recommender systems~\citep{GRU4Rec,HGRU,TCN,li2017neural,liu2018stamp,wu2019session,HRNN}, GRU4Rec~\citep{GRU4Rec} was first proposed to use gated recurrent units (GRU)~\citep{GRU}, an RNN variant with gating mechanism, for recommendation. Since then, follow-up works such as hierarchical GRU~\citep{HGRU}, temporal convolutional networks (TCN)~\citep{TCN}, and hierarchical RNN (HRNN)~\citep{HRNN} have tried to achieve improvement in accuracy with the help of cross-session information~\citep{HGRU}, causal convolutions~\citep{TCN}, as well as control signals~\citep{HRNN}. We note that our REN does not assume specific RNN architectures (e.g., GRU or TCN) and is therefore \emph{compatible with different RNN-based (or more generally DL-based) models}, as shown in later sections.

\textbf{Contextual Bandits.}
Contextual bandit algorithms such as LinUCB~\citep{li2010contextual} and its variants~\citep{LSB,taming,CFBandit,LRBandit,practical,korda2016distributed,mahadik2020fast,NeuralUCB} have been proposed to tackle the exploitation-exploration trade-off in recommender systems and successfully improve upon context-free bandit algorithms~\citep{LinRel}. 
Similar to~\cite{LinRel}, theoretical analysis shows that LinUCB variants could achieve a rate-optimal regret bound~\citep{chu2011contextual}. However, these methods either assume observed context~\citep{NeuralUCB} or are incompatible with neural networks~\citep{CFBandit,LSB}. In contrast, REN as a contextual bandit algorithm runs in the latent space and assumes user models based on RNN; therefore it is compatible with state-of-the-art RNN-based recommender systems.

% A different approach focuses on using black-box policies to abstract away the learning problem \citep{taming,practical}. The idea is to characterise the regret bounds by the number of these policies, each of which may be generalized to a deep model with fixed parameters. However, it is rather challenging to map all deep models with continuous parameters to the finite policy set. In this aspect, our analysis explicitly handles the approximation of embedding-based models in continuous spaces.

\textbf{Diversity-Inducing Models.}
Various works have focused on inducing diversity in recommender systems~\citep{nguyen2014exploring,antikacioglu2017post,wilhelm2018practical,bello2018seq2slate}. Usually such a system consists of a submodular function, which measures the diversity among items, and a relevance prediction model, which predicts relevance between users and items. Examples of submodular functions include the probabilistic coverage function~\citep{CLSB} and facility location diversity (FILD)~\citep{FILD}, while relevance prediction models can be Gaussian processes~\citep{DivGP}, linear regression~\citep{LSB}, etc. These models typically focus on improving \emph{diversity among recommended items in a slate} at the cost of accuracy. In contrast, REN's goal is to optimize for long-term rewards through improving \emph{diversity between previous and recommended items}.  
% {\color{blue}\st{without sacrificing accuracy} 
We include some slate generation in our real-data experiments for completeness.%} %sssssssssssssss ssssss sssss

%Although REN naturally connects to diversity, here we note several key differences from the methods above: (1) These models typically focus on improving diversity at the cost of accuracy. In contrast, REN's goal is to optimize for long-term rewards, i.e., achieving low regret. (2) These models 

% sometimes also connect to bandit

\section{Recurrent Exploration Networks}\label{sec:rdn}
In this section we first describe the general notations and how RNN can be used for recommendation, briefly review determinantal point processes (DPP) as a diversity-inducing model as well as their connection to exploration in contextual bandits, and then introduce our proposed REN framework.

\subsection{Notation and RNN-Based Recommender Systems}\label{sec:notation}
\textbf{Notation.} We consider the problem of sequential recommendations where the goal is to predict the item a user interacts with (e.g., click or purchase) at time $t$, denoted as $\e_{k_t}$, given her previous interaction history $\E_t = [\e_{k_\tau}]_{\tau=1}^{t-1}$. Here $k_t$ is the index for the item at time $t$, $\e_{k_t}\in\{0,1\}^K$ is a one-hot vector indicating an item, and $K$ is the number of total items. We denote the item embedding (encoding) for $\e_{k_t}$ as $\x_{k_t} = f_e(\e_{k_t})$, where $f_e(\cdot)$ is the encoder as a part of the RNN. Correspondingly we have $\X_t = [\x_{k_\tau}]_{\tau=1}^{t-1}$. Strictly speaking, in an online setting where the model updates at every time step $t$, $\x_k$ also changes over time; in~\secref{sec:rdn} we use $\x_k$ as a shorthand for $\x_{t,k}$ for simplicity. We use $\|\z\|_{\infty} = \max_i |\z^{(i)}|$ to denote the $L_{\infty}$ norm, where the superscript $(i)$ means the $i$-th entry of the vector $\z.$

\textbf{RNN-Based Recommender Systems.}
Given the interaction history $\E_t$, the RNN generates the user embedding at time t as $\tha_t = R([\x_{k_\tau}]_{\tau=1}^{t-1})$, where $\x_{k_\tau} = f_e(\e_{k_\tau})\in\mathbb{R}^d$, and $R(\cdot)$ is the recurrent part of the RNN. Assuming tied weights, the score for each candidate item is then computed as $p_{k,t} = \x_k^\top\tha_t$. As the last step, the recommender system will recommend the items with the highest scores to the user. Note that the subscript $k$ indexes the items, and is equivalent to an `action', usually denoted as $a$, in the context of bandit algorithms.

\subsection{Determinantal Point Processes for Diversity and Exploration}

Determinantal point processes (DPP) consider an item selection problem where each item is represented by a feature vector $\x_t$.
%, $\psi_j$ with unit norm $\|\psi_j\|_2=1$.
Diversity is achieved by picking a subset of items to cover the maximum volume spanned by the items, measured by the log-determinant of the corresponding kernel matrix,
${\rm ker}(\X_t) = \log\det(\I_K + \X_t \X_t^\top)$, where $\I_K$ is included to prevent singularity. Intuitively, DPP penalizes colinearity, which is an indicator that the topics of one item are already covered by the other topics in the full set.
The log-determinant of a kernel matrix is also a submodular function \citep{friedland2013submodular}, which implies a $(1-1/e)$-optimal guarantees from greedy solutions. 
The greedy algorithm for DPP via the matrix determinant lemma is
\begin{align}
    \argmax\nolimits_{k}\; 
    &\log\det(\I_d + \X_{t}^\top \X_{t} + \x_k\x_k^\top)\\
    &-\log\det(\I_d + \X_{t}^\top \X_{t})
    \nonumber \\
    =
    \argmax\nolimits_{k}\; &\log(
    1+\x_k^\top (\I_d + \X_{t}^\top \X_{t})^{-1}\x_k
    )\\
    = \argmax\nolimits_{k}\; &\sqrt{\x_k^\top (\I_d + \X_{t}^\top \X_{t})^{-1}\x_k}.
    \label{eq:salience}
\end{align}

Interestingly, note that $\sqrt{\x_k^\top (\I_d + \X_{t}^\top \X_{t})^{-1}\x_k}$ has the same form as the confidence interval in LinUCB~\citep{li2010contextual}, a commonly used contextual bandit algorithm to boost exploration and achieve long-term rewards, suggesting a connection between diversity and long-term rewards~\citep{LSB}. Intuitively, this makes sense in recommender systems since encouraging diversity relative to user history 
% (\emph{{\color{blue}\st{not} as well as} diversity in a slate of recommendations {\color{blue} in our experiments}}) 
(\emph{as well as diversity in a slate of recommendations in our experiments}) 
naturally explores user interest previously unknown to the model, leading to much higher long-term rewards, as shown in \secref{sec:experiment}.
%recommending diverse items to user naturally explores the user's interest that is previously unknown to the model. This may sacrifice short-term accuracy, but can lead to much higher long-term rewards, as shown in \secref{sec:experiment}.

% Should be \log(1+x^\top (\Lambda + X_k^\top X_k)^{-1}x) in the last line?
% We may redefine the objective by the Mahalanobis salience given the current observations, 
% $\argmax_x\|(\Lambda+X_k^\top X_k)^{-\frac{1}{2}}x\|_2$, which connects to set-cover objectives with $X=H$ and a different modulo definition.
% However, modeling the problem with Mahalanobis salience allow end-to-end training where $x$ is replaced by a latent embedding vector $w(x)$.

\subsection{Recurrent Exploration Networks}\label{sec:rdn_final}

\textbf{Exploration Term.} Based on the intuition above, we can modify the user-item score $p_{k,t} = \x_k^\top\tha_t$ to include a diversity (exploration) term, leading to the new score
\begin{align}
p_{k,t} = \x_k^\top\tha_t + \lambda_d \sqrt{\x_k^\top (\I_d + \X_{t}^\top \X_{t})^{-1}\x_k} \label{eq:rdn_wo_uncertainty},
\end{align}
where the first term is the relevance score and the second term is the exploration score (measuring diversity \emph{between previous and recommended items}). $\tha_t = R(\X_t) = R([\x_{k_\tau}]_{\tau=1}^{t-1})$ is RNN's hidden states at time $t$ representing the user embedding. The hyperparameter $\lambda_d$ aims to balance two terms.

% \begin{algorithm}[t]\label{alg:rdn}
% \caption{Recurrent Exploration Networks (REN)}
% \SetAlgoLined
% \textbf{Input:} $\lambda_d$, $\lambda_u$, initialized REN model with the encoder, i.e., $R(\cdot)$ and $f_e(\cdot)$.\\
% \For{$t = 1,2,\dots,T$}{
% Obtain item embeddings from REN: $\muu_{k_\tau} \gets f_e(\e_{k_\tau})$ for all $\tau\in\{1, 2, \dots, t-1\}$.\\
% Obtain the current user embedding from REN: $\tha_t\gets R(\D_t)$.\\
% $\A_t \gets \I_d + \sum_{\tau\in \Psi_t}\muu_{k_\tau}^\top \muu_{k_\tau}$.\\
% Obtain candidate items' embeddings from REN: $\muu_{k}\gets f_e(\e_{k})$, where $k\in [K]$.\\
% Obtain candidate items' uncertainty estimates $\si_{k}$, where $k\in [K]$.\\
% \For{$k\in [K]$}{
% $p_{k,t} \gets \muu_k^\top\tha_t + \lambda_d \sqrt{\muu_k^\top (\I_d + \D_{t}^\top \D_{t})^{-1}\muu_k} + \lambda_u \|\si_k\|_{\infty}.$
% } % end of for
% Recommend item $k_t \gets \argmax_k p_{t,k}$ and collect user feedbacks. \\
% Update the REN model $R(\cdot)$ and $f_e(\cdot)$ using collected user feedbacks. \label{algl:update_rdn}
% } % end of big for
% \end{algorithm}

\begin{algorithm}[t]
\caption{Recurrent Exploration Networks (REN)}\label{alg:rdn}
\SetAlgoLined
\textbf{Input:} $\lambda_d$, $\lambda_u$, initialized REN model with the encoder, i.e., $R(\cdot)$ and $f_e(\cdot)$.\\
\For{$t = 1,2,\dots,T$}{
Obtain item embeddings from REN:

\hskip1.5em $\muu_{k_\tau} \gets f_e(\e_{k_\tau})$ for all $\tau\in\{1, 2, \dots, t-1\}$.\\
Obtain the current user embedding from REN: 

\hskip1.5em $\tha_t\gets R(\D_t)$.\\
Compute $\A_t \gets \I_d + \sum_{\tau\in \Psi_t}\muu_{k_\tau}^\top \muu_{k_\tau}$.\\
Obtain candidate items' embeddings from REN:

\hskip1.5em $\muu_{k}\gets f_e(\e_{k})$, where $k\in [K]$.\\
Obtain candidate items' uncertainty estimates $\si_{k}$, where $k\in [K]$.\\
\For{$k\in [K]$}{
Obtain the score for item $k$ at time $t$:

\hskip1.5em $p_{k,t} \gets \muu_k^\top\tha_t + \lambda_d \sqrt{\muu_k^\top \A_t^{-1}\muu_k} + \lambda_u \|\si_k\|_{\infty}.$
} % end of for
Recommend item $k_t \gets \argmax_k p_{t,k}$ and collect user feedbacks. \\
Update the REN model $R(\cdot)$ and $f_e(\cdot)$ using collected user feedbacks. \label{algl:update_rdn}
} % end of big for
\end{algorithm}

\textbf{Uncertainty Term for Context Uncertainty.} At first blush, given the user history the system using \eqnref{eq:rdn_wo_uncertainty} will recommend items that are (1) relevant to the user's interest and (2) diverse from the user's previous items. However, this only works when item embeddings $\x_k$ are correctly learned. Unfortunately, the quality of learned item embeddings, in turn, relies heavily on the effectiveness of exploration, leading to a chicken-and-egg problem. To address this problem, one also needs to consider the uncertainty of the learned item embeddings. Assuming the item embedding $\x_k\sim\NM(\muu_k, \Si_k)$, where $\Si_k = \mathbf{diag}(\si_k^2)$, we have the final score for REN:
\begin{align}
p_{k,t} = \muu_k^\top\tha_t + \lambda_d \sqrt{\muu_k^\top (\I_d + \D_{t}^\top \D_{t})^{-1}\muu_k} + \lambda_u \|\si_k\|_{\infty},  \label{eq:rdn_w_uncertainty}
\end{align}
where $\tha_t = R(\D_t) = R([\muu_{k_\tau}]_{\tau=1}^{t-1})$ and $\D_t = [\muu_{k_\tau}]_{\tau=1}^{t-1}$. The term $\si_k$ quantifies the uncertainty for each dimension of $\x_k$, meaning that items whose embeddings REN is uncertain about are more likely to be recommended. Therefore with the third term, REN can naturally balance among relevance, diversity (\emph{relative to user history}), and uncertainty during exploration.

\textbf{Putting It All Together.} \algref{alg:rdn} shows the overview of REN. Note that the difference between REN and traditional RNN-based recommenders is only in the inference stage. During training (Line~\ref{algl:update_rdn} of \algref{alg:rdn}), one can train REN only with the relevance term using models such as GRU4Rec and HRNN. In the experiments, we use uncertainty estimates $\mathbf{diag}(\si_{k}) = 1/\sqrt{n_{k}}\;\I_d$, where $n_{k}$ is item $k$'s total number of impressions (i.e., the number of times item $k$ has been recommended) for all users. The intuition is that: the more frequently item $k$ is recommended, the more frequently its embedding $\x_k$ gets updated, the faster $\si_{k}$ decreases.%
\footnote{There are some caveats in general. $\mathbf{diag}(\si_{k})\propto\I_d$ assumes that all coordinates of $x$ shrink at the same rate. However, REN exploration mechanism associates $n_k$ with the total variance of the features of an item. This may not ensure all feature dimensions to be equally explored. See \citep{pmlr-v97-jun19a} for a different algorithm that analyzes the exploration of the low-rank feature space.}
Our preliminary experiments show that $1/\sqrt{n_{k}}$ does decrease at the rate of $O(1/\sqrt{t})$, meaning that the assumption in \lemref{lem:sigma_to_sqrt} is satisfied. 
From the Bayesian perspective, $1/\sqrt{n_{k}}$ may not accurately reflect the uncertainty of the learned $x_k$, which is a limitation of our model. 
In principle, one can learn $\si_k$ from data using the reparameterization trick~\citep{VAE} with a Gaussian prior on $x_k$ and examine whether $\si_k$ the assumption in \lemref{lem:sigma_to_sqrt}; this would be interesting future work. %Once the model is trained, inference can be performed as in~\algref{alg:rdn}.

\textbf{Linearity in REN. }
REN only needs a linear bandit model; REN's output $\x_k^\top \tha_t$ is linear w.r.t. $\tha$ and $\x_k$. Note that NeuralUCB~\citep{NeuralUCB} is a powerful nonlinear extension of LinUCB, i.e., its output is nonlinear w.r.t. $\tha$ and $\x_k$. Extending REN's output from $\x_k^\top \tha_t$ to a nonlinear function $f(\x_k, \tha_t)$ as in NeuralUCB is also interesting future work.%
\footnote{In other words, we did not fully explain why $x$ could be shared between non-linear RNN and the uncertainty bounds based on linear models. On the other hand, we did observe promising empirical results, which may encourage interested readers to dive deep into different theoretical analyses.}

\textbf{Beyond RNN.} Note that our methods and theory go beyond RNN-based models and can be naturally extended to any latent factor models including transformers, MLPs, and matrix factorization. The key is the user embedding $\tha_t=R(\X_t)$, which can be instantiated with an RNN, a transformer, or a matrix-factorization model.

% \textbf{RNN to Learn $\tha_t$.} A linear RNN with tied weights and a single time step is equivalent to linear regression (LR); therefore RNN is a more general model to learn $\tha_t$. Compared to LR, RNN-based recommenders can naturally incorporate new user history by incrementally updating the hidden states ($\tha_t$ in REN), without the need to solve a linear equation. Interestingly, one can also see RNN's recurrent computation as a simulation (approximation) for solving equations via iterative updating. 

% , while NeuralUCB's output is nonlinear w.r.t. to $\theta$ and x. (2) NeuralUCB is a powerful nonlinear extension of LinUCB; its goal is orthogonal to REN whose reward estimate is linear. Extending REN's output from $\theta^\top x$ to a nonlinear function $f(x,\theta)$ as in NeuralUCB is interesting future work. (3) In contrast, REN's goal is to handle uncertainty in item embeddings for the linear bandit setting. Extending such analysis to the nonlinear setting requires insight from NeuralUCB. 

\section{Theoretical Analysis}
With REN's connection to contextual bandits, we can prove that with proper $\lambda_d$ and $\lambda_u$, \eqnref{eq:rdn_w_uncertainty} is actually the upper confidence bound that leads to long-term rewards with a rate-optimal regret bound.% Specifically, theory, confidence bound, regret bound

\textbf{Reward Uncertainty versus Context Uncertainty.} Note that unlike existing works which primarily consider the randomness from the reward, we take into consideration the uncertainty resulted from the context (content)~\cite{mi2019training,NPN}, i.e., \emph{context uncertainty}. In CDL~\cite{CDL,CRAE}, it is shown that such content information is crucial in DL-based RecSys~\cite{CDL,CRAE}, and so is the associated uncertainty. More specifically, existing works assume deterministic $\x$ and only assume randomness in the reward, i.e., they assume that $r = \x^\top \tha + \epsilon$, and therefore $r$’s randomness is independent of $\x$. The problem with this formulation is that they assume $\x$ is deterministic and therefore the model only has a point estimate of the item embedding $\x$, but does not have uncertainty estimation for such $\x$. We find that such uncertainty estimation is crucial for exploration; if the model is uncertain about $\x$, it can then explore more on the corresponding item. 

To facilitate analysis, we follow common practice~\citep{LinRel,chu2011contextual} to divide the procedure of REN into ``BaseREN" (\algref{alg:baserdn}) and ``SupREN" stages correspondingly. Essentially SupREN introduces $S = \ln T$ levels of elimination (with $s$ as an index) to filter out low-quality items and ensures that the assumption holds (see the Supplement for details of SupREN). 

% To facilitate analysis, we follow the BaseLinUCB-SupLinUCB decomposition of LinUCB~\citep{chu2011contextual} and divide the procedure of REN into ``BaseREN" (\algref{alg:baserdn}) and ``SupREN" stages correspondingly. Essentially SupREN introduces $S = \ln T$ levels of elimination (with $s$ as an index) to filter out low-quality items and ensures that the assumption holds (see the Supplement for details of SupREN). 

In this section, we first provide a high probability bound for BaseREN with uncertain embeddings (context), and derive an upper bound for the regret. As mentioned in \secref{sec:notation}, for the online setting where the model updates at every time step $t$, $\x_k$ also changes over time. Therefore in this section we use $\x_{t,k}$, $\muu_{t,k}$, $\Si_{t,k}$, and $\si_{t,k}$ in place of $\x_k$, $\muu_{k}$, $\Si_{k}$, and $\si_{k}$ from \secref{sec:rdn} to be rigorous.

% Essentially SupREN introduce $S = \ln T$ levels of elimination (with $s$ as an index) to filter out low-quality items and ensure that the assumption holds (see the Supplement for details of SupREN). 

\begin{assup}
Assume there exists an optimal $\tha^*$, with $\|\tha^*\| \leq 1$, and $\x_{t,k}^*$ such that $\E[r_{t,k}] = {\x_{t,k}^*}^\top \tha^*$.  Further assume that there is an effective distribution $\NM(\muu_{t,k}, \Si_{t,k})$ such that ${\x}_{t,k}^* \sim \NM(\muu_{t,k}, \Si_{t,k})$ where $\Si_{t,k} = \textbf{diag}(\si_{t,k}^2)$. Thus, the true underlying context is unavailable, but we are aided with the knowledge that it is generated by a multivariate normal with known parameters\footnote{Here we omit the identifiability issue of ${\x}_{t,k}^*$ and assume that there is a unique ${\x}_{t,k}^*$ for clarity. }. 
\end{assup}
% removed the word `posterior' to avoid attacks from Bayesian researchers

% \bernie{make the notation consistent, $\bm{x}$ for vector, $\mathbf{D}$ for matrix, $s$ for scala.}
% \hw{HW: done:) let me know if I missed anything}

% \hw{HW: Currently there is some disconnect between $\tha$ in LinUCB, which is updated online as a linear model, and $\tha$ in REN, which is from REN. Need to clarity this somewhere.}

\begin{algorithm}[t]
\caption{BaseREN: Basic REN Inference at Step $t$}\label{alg:baserdn}
\SetAlgoLined
\textbf{Input:} $\alpha$, $\Psi_t\subseteq \{1,2,\dots,t-1\}$.\\
Obtain item embeddings from REN: $\muu_{\tau,k_\tau} \gets f_e(\e_{\tau,k_\tau})$ for all $\tau\in\Psi_t$.\\
Obtain user embedding: $\tha_t\gets R(\D_t)$.\\
$\A_t \gets \I_d + \sum_{\tau\in \Psi_t}\muu_{\tau,k_\tau}^\top \muu_{\tau,k_\tau}$.\\
Obtain candidate items' embeddings: $\muu_{t,k}\gets f_e(\e_{t,k})$, where $k\in [K]$.\\
Obtain candidate items' uncertainty estimates $\si_{t,k}$, where $k\in [K]$.\\
\For{$a\in [K]$}{
$s_{t,k} = \sqrt{\muu_{t,k}^\top \A_{t}^{-1} \muu_{t,k}}$\\
$w_{t,k} \gets (\alpha+1)s_{t,k} + (4\sqrt{d} + 2\sqrt{\ln \frac{TK}{\delta}})\|\si_{t,k}\|_{\infty}$.\label{algl:width}\\
$\widehat{r}_{t,k} \gets \tha_t^\top \muu_{t,k}$.
} % end of for
Recommend item $k\gets \argmax_k \widehat{r}_{t,k} + w_{t,k}$.
\end{algorithm}

\subsection{Upper Confidence Bound for Uncertain Embeddings}\label{sec:confidence_bound}

For simplicity denote the item embedding (context) as $\x_{t,k}$, where $t$ indexes the rounds (time steps) and $k$ indexes the items. We define:
\begingroup\makeatletter\def\f@size{9}\check@mathfonts
\begin{align}
s_{t,k} &= \sqrt{\muu_{t,k}^\top \A_{t}^{-1} \muu_{t,k}} \in \mathbb{R}_+, \;\;
\D_t = [\muu_{\tau,k_{\tau}}]_{\tau \in \Psi_t} \in \mathbb{R}^{|\Psi_t|\times d}, \nonumber\\
\y_t &= [r_{\tau,k_{\tau}}]_{\tau \in \Psi_t} \in \mathbb{R}^{|\Psi_t|\times 1}, \;\;
\A_t = \I_d + \D_t^\top \D_t, \nonumber\\
\b_t &= \D_t^\top \y_t, \;\;\;\;\;\;\;\;\;\;\;\;\;\;\;\;\;\;\;\;\;\;\;\;
% \hat{\tha} &= R(\X_t) \\
\widehat{r}_{t,k} = \muu_{t,k}^\top \hat{\tha_t} = \muu_{t,k}^\top \A_t^{-1} \b_t,\label{eq:to_estimate}
\end{align}
\endgroup
where $\y_t$ is the collected user feedback. \lemref{lem:simple_cb} below shows that with $\lambda_d=1+\alpha=1+\sqrt{\frac{1}{2}\ln \frac{2TK}{\delta}}$ and $\lambda_u=4\sqrt{d} + 2\sqrt{\ln \frac{TK}{\delta}}$, \eqnref{eq:rdn_w_uncertainty} is the upper confidence bound with high probability, meaning that \eqnref{eq:rdn_w_uncertainty} upper bounds the true reward with high probability, which makes it a reasonable score for recommendations. % and $\hat{\tha}$ is the estimated user embedding from REN.
% \begin{align*}
% s_{t,k} &= \sqrt{\muu_{t,k}^\top \A_{t}^{-1} \muu_{t,k}} \in \mathbb{R}_+ \\
% \D_t &= [\muu_{\tau,k_{\tau}}]_{\tau \in \Psi_t} \in \mathbb{R}^{|\Psi_t|\times d} \\
% \y_t &= [r_{\tau,k_{\tau}}]_{\tau \in \Psi_t} \in \mathbb{R}^{|\Psi_t|\times 1} \\
% \A_t &= \I_d + \D_t^\top \D_t \\
% \b_t &= \D_t^\top \y_t \\
% % \hat{\tha} &= R(\X_t) \\
% \widehat{r}_{t,k} &= \muu_{t,k}^\top \hat{\tha} = \muu_{t,k}^\top \A_t^{-1} \b_t
% \end{align*}
% \begin{align*}
% s_{t,k} = \sqrt{x_{t,k}^\top A_{t}^{-1} x_{t,k}} \in \mathbb{R}_+
% \end{align*}

\begin{lemma}[\textbf{Confidence Bound}]\label{lem:simple_cb}
With probability at least $1 - 2 \delta / T$, we have for all $k\in[K]$ that
\begin{align*}
|\widehat{r}_{t,k} - {\x_{t,k}^*}^\top \tha^*| \leq &(\alpha + 1)s_{t,k} \\
&+ (4\sqrt{d} + 2\sqrt{\ln \frac{TK}{\delta}})\|\si_{t,k}\|_{\infty},
\end{align*}
where $\|\si_{t,k}\|_{\infty} = \max_i |\si_{t,k}^{(i)}|$ is the $L_{\infty}$ norm.
\end{lemma}

The proof is in the Supplement. This upper confidence bound above provides important insight on why \eqnref{eq:rdn_w_uncertainty} is reasonable as a final score to select items in~\algref{alg:rdn} as well as the choice of hyperparameters $\lambda_d$ and $\lambda_u$.
%keep it as `Supplement' to be consistent with the other places to save space: )

\textbf{RNN to Estimate $\tha_t$.} 
REN uses RNN to approximate $\A_t^{-1} \b_t$ (useful in the proof of~\lemref{lem:simple_cb}) in \eqnref{eq:to_estimate}. Note that a linear RNN with tied weights and a single time step is equivalent to linear regression (LR); therefore RNN is a more general model to estimate $\tha_t$. Compared to LR, RNN-based recommenders can naturally incorporate new user history by incrementally updating the hidden states ($\tha_t$ in REN), without the need to solve a linear equation. 
Interestingly, one can also see RNN's recurrent computation as a simulation (approximation) for solving equations via iterative updating.

\subsection{Regret Bound}
\lemref{lem:simple_cb} above provides an estimate of the reward's upper bound at time $t$. Based on this estimate, one natural next step is to analyze the regret after all $T$ rounds. Formally, we define the regret of the algorithm after $T$ rounds as
\begin{align}\label{eq:regret}
B(T) = \sum_{t=1}^T r_{t,k_t^*} - \sum_{t=1}^T r_{t,k_t},
\end{align}
where $k_t^*$ is the optimal item (action) $k$ at round $t$ that maximizes $\E[r_{t,k}] = {\x_{t,k}^*}^\top\tha^*$, and $k_t$ is the action chose by the algorithm at round $t$. Similar to~\citep{LinRel}, SupREN calls BaseREN as a sub-routine. In this subsection, we derive the regret bound for SupREN with uncertain item embeddings.

\begin{lemma}\label{lem:regret_to_width}
With probability $1-2\delta S$, for any $t\in [T]$ and any $s\in [S]$, we have: (1) $|\widehat{r}_{t,k} - \E[r_{t,k}]|\leq w_{t,k}$ for any $k\in[K]$, (2) $k_t^*\in \hat{A}_s$, and (3) $\E[r_{t,k^*_t}] - \E[r_{t,k}] \leq 2^{(3-s)}$ for any $k\in\hat{A}_s$.
\end{lemma}
% \begin{proof}
% The proof is a simple modification of that in \citeauthor{LinRel} [2002, Lemma 15] to accommodate modification in \lemref{lem:simple_cb}.
% \end{proof}

% \begin{lemma}\label{lem:s_to_sqrt}%[\citealt{chu2011contextual}, Lemma 6]
% In BaseREN, we have
% \begin{align*}
% (1+\alpha)\sum_{t\in\Psi_{T+1}} s_{t,k_t} \leq 5 \cdot (1+\alpha^2) \sqrt{d|\Psi_{T+1}|}.
% \end{align*}
% \end{lemma}
\begin{lemma}\label{lem:s_to_sqrt}%[\citealt{chu2011contextual}, Lemma 6]
In BaseREN, we have: 
$
(1+\alpha)\sum_{t\in\Psi_{T+1}} s_{t,k_t} \leq 5 \cdot (1+\alpha^2) \sqrt{d|\Psi_{T+1}|}.
$
\end{lemma}
% \begin{proof}
% This is a direct result of Lemma 3 and Lemma 6 in \cite{chu2011contextual} as well as Lemma 16 in \cite{LinRel}.
% \end{proof}

% \begin{lemma}\label{lem:sigma_to_sqrt}
% Assuming $\|\si_{1,k}\|_{\infty} = 1$ and $\|\si_{t,k}\|_{\infty} \leq \frac{1}{\sqrt{t}}$ for any $k$ and $t$, then for any $k$,
% \begin{align*}
% \sum_{t\in\Psi_{T+1}} \|\si_{t,k}\|_{\infty} \leq \sqrt{|\Psi_{T+1}|}.
% \end{align*}
% \end{lemma}
\begin{lemma}\label{lem:sigma_to_sqrt}
Assuming $\|\si_{1,k}\|_{\infty} = 1$ and $\|\si_{t,k}\|_{\infty} \leq \frac{1}{\sqrt{t}}$ for any $k$ and $t$, then for any $k$, we have the upper bound: 
$
\sum_{t\in\Psi_{T+1}} \|\si_{t,k}\|_{\infty} \leq \sqrt{|\Psi_{T+1}|}.
$
\end{lemma}
% \begin{proof}
% Since the function $f(t) = \frac{1}{\sqrt{t}}$ is convex when $t>0$, we have 
% \begin{align*}
% \sum_{t=1}^{|\Psi_{T+1}|}\frac{1}{\sqrt{t}}\leq \int_0^{|\Psi_{T+1}|} \frac{1}{\sqrt{t}} =\left. \sqrt{t} \right\vert_0^{|\Psi_{T+1}|} = \sqrt{|\Psi_{T+1}|}
% \end{align*}
% \end{proof}
Essentially \lemref{lem:regret_to_width} links the regret $B(T)$ to the width of the confidence bound $w_{t,k}$ (Line~\ref{algl:width} of \algref{alg:baserdn} or the last two terms of \eqnref{eq:rdn_w_uncertainty}). \lemref{lem:s_to_sqrt} and \lemref{lem:sigma_to_sqrt} then connect $w_{t,k}$ to $\sqrt{|\Psi_{T+1}|}\leq \sqrt{T}$, which is sublinear in $T$; this is the key to achieve a sublinear regret bound. Note that $\hat{A}_s$ is defined inside Algorithm 2 (SupREN) of the Supplement.

Interestingly, \lemref{lem:sigma_to_sqrt} states that the uncertainty only needs to decrease at the rate $\frac{1}{\sqrt{t}}$, which is consistent with our choice of $\mathbf{diag}(\si_{k}) = 1/\sqrt{n_{k}}\;\I_d$ in~\secref{sec:rdn_final}, where $n_{k}$ is item $k$'s total number of impressions for all users. As the last step, \lemref{lem:psi_to_sqrt} and \thmref{thm:regret_bound} below build on all lemmas above to derive the final sublinear regret bound. 

% this is the one-column version
\begin{lemma}\label{lem:psi_to_sqrt}
For all $s\in [S]$,
\begingroup\makeatletter\def\f@size{9}\check@mathfonts
\def\maketag@@@#1{\hbox{\m@th\large\normalfont#1}}%
\begin{align*}
|\Psi_{T+1}^{(s)}| \leq 2^s \cdot (5(1+\alpha^2)\sqrt{d|\Psi_{T+1}^{(s)}|}
+ 4\sqrt{dT} + 2\sqrt{T\ln \frac{TK}{\delta}}).
\end{align*}
\endgroup
\end{lemma}

\begin{theorem}\label{thm:regret_bound}
If SupREN is run with $\alpha=\sqrt{\frac{1}{2}\ln \frac{2TK}{\delta}}$, with probability at least $1 - \delta$, the regret of the algorithm is
% \begin{align}
% O\left(\sqrt{Td\ln^3\left(\frac{KT\ln(T)}{\delta}\right)}\right).
% \end{align}
% this is the one-column version
\begin{align*}
B(T) &\leq 2\sqrt{T} + 92\cdot(1+\ln \frac{2TK(2\ln T+2)}{\delta})^{\frac{3}{2}}\sqrt{Td}\\
&=O(\sqrt{Td\ln^3(\frac{KT\ln(T)}{\delta}})),
\end{align*}
% this is the two-column version
% \begin{align*}
% B(T)&\leq 2\sqrt{T} + 92\cdot(1+\ln \frac{2TK(2\ln T+2)}{\delta})^{\frac{3}{2}}\sqrt{Td}\\
% &=O(\sqrt{Td\ln^3(\frac{KT\ln(T)}{\delta}})),
% \end{align*}
\end{theorem}

The full proofs of all lemmas and the theorem are in the Supplement. \thmref{thm:regret_bound} shows that even with the uncertainty in the item embeddings (i.e., context uncertainty), our proposed REN can achieve the same rate-optimal sublinear regret bound.% as in~\citet{chu2011contextual}.

\section{Experiments}\label{sec:experiment}
In this section, we evaluate our proposed REN on both synthetic and real-world datasets. 

\subsection{Experiment Setup and Compared Methods}\label{sec:setup}
\textbf{Joint Learning and Exploration Procedure in Temporal Data.}
To effectively verify REN's capability to boost long-term rewards, we adopt an online experiment setting where data is divided into different time intervals $[T_0, T_1), [T_1, T_2), \dots, [T_{M-1}, T_M]$. RNN (including REN and its baselines) is then trained and evaluated in a rolling manner: (1) RNN is trained using data in $[T_0, T_1)$; (2) RNN is evaluated using data in $[T_1, T_2)$ and collects feedbacks (rewards) for its recommendations; (3) RNN uses newly collected feedbacks from $[T_1, T_2)$ to finetune the model; (4) Repeat the previous two steps using data from the next time interval. Note that different from traditional offline and one-step evaluation, corresponding to only Step (1) and (2), our setting performs joint learning and exploration in temporal data, and therefore is more realistic and closer to production systems.
% HOG: can be shortened here

\textbf{Long-Term Rewards.} Since the goal is to evaluate long-term rewards, we are mostly interested in the rewards during the last (few) time intervals. Conventional RNN-based recommenders do not perform exploration and are therefore much easier to saturate at a relatively low reward. In contrast, REN with its effective exploration can achieve nearly optimal rewards in the end.

% \begin{figure*}[!tb]
% \begin{center}
% %\framebox[4.0in]{$\;$}
% %\includegraphics[height=5cm]{likeli1.eps}
% \hskip-0.1cm
% \subfigure{
% \includegraphics[width=0.325\textwidth]{figures/sensitivity-r-f1}}
% \subfigure{
% \includegraphics[width=0.325\textwidth]{figures/sensitivity-r-f10}}
% \subfigure{
% \includegraphics[width=0.325\textwidth]{figures/sensitivity-r-f50}}
% \end{center}
% \vskip -0.2in
% \caption{\label{fig:sensitivity}Hyperparameter sensitivity for $\lambda_d$ in \emph{SYN-S}, \emph{SYN-M}, and \emph{SYN-L}.
% }
% \vskip -0.2in
% \end{figure*}

\textbf{Compared Methods.}
We compare REN variants with state-of-the-art RNN-based recommenders including \textbf{GRU4Rec}~\citep{GRU4Rec}, \textbf{TCN}~\citep{TCN}, \textbf{HRNN}~\citep{HRNN}. Since REN can use any RNN-based recommenders as a base model, we evaluate three REN variants in the experiments: \textbf{REN-G}, \textbf{REN-T}, and \textbf{REN-H}, which use GRU4Rec, TCN, and HRNN as base models, respectively. Additionally we also evaluate \textbf{REN-1,2}, an REN variant without the third term of \eqnref{eq:rdn_w_uncertainty}, and \textbf{REN-1,3}, one without the second term of \eqnref{eq:rdn_w_uncertainty}, as an ablation study. Both REN-1,2 and  REN-1,3 use GRU4Rec as the base model. As references we also include \textbf{Oracle}, which always achieves optimal rewards, and \textbf{Random}, which randomly recommends one item from the full set. For REN variants we choose $\lambda_d$ from $\{0.001, 0.005, 0.01, 0.05, 0.1\}$ and set $\lambda_u = \sqrt{10}\lambda_d$. Other hyperparameters in the RNN base models are kept the same for fair comparison (see the Supplement for more details on neural network architectures, hyperparameters, and their sensitivity analysis).

\textbf{Connection to Reinforcement Learning (RL) and Bandits.} REN-1,2 (in \figref{fig:ablation}) can be seen as a simplified version of `randomized least-squares value iteration' (an RL approach proposed in~\cite{osband2016generalization}) or an adapted version of contextual bandits, while REN-1,3 (in \figref{fig:ablation}) is an advanced version of $\epsilon$-greedy exploration in RL. Note that REN is orthogonal to RL~\cite{shi2019virtual} and bandit methods. % consider deleting
%We will include such discussion in the revision as R2 suggested. (3) Thorough study of RL methods is interesting future work, but out of the scope of this paper.

\begin{figure*}[!tb]
\begin{center}
% \vskip -0.2in
%\framebox[4.0in]{$\;$}
%\includegraphics[height=5cm]{likeli1.eps}
\hskip-0.1cm
\subfigure{
\includegraphics[width=0.325\textwidth]{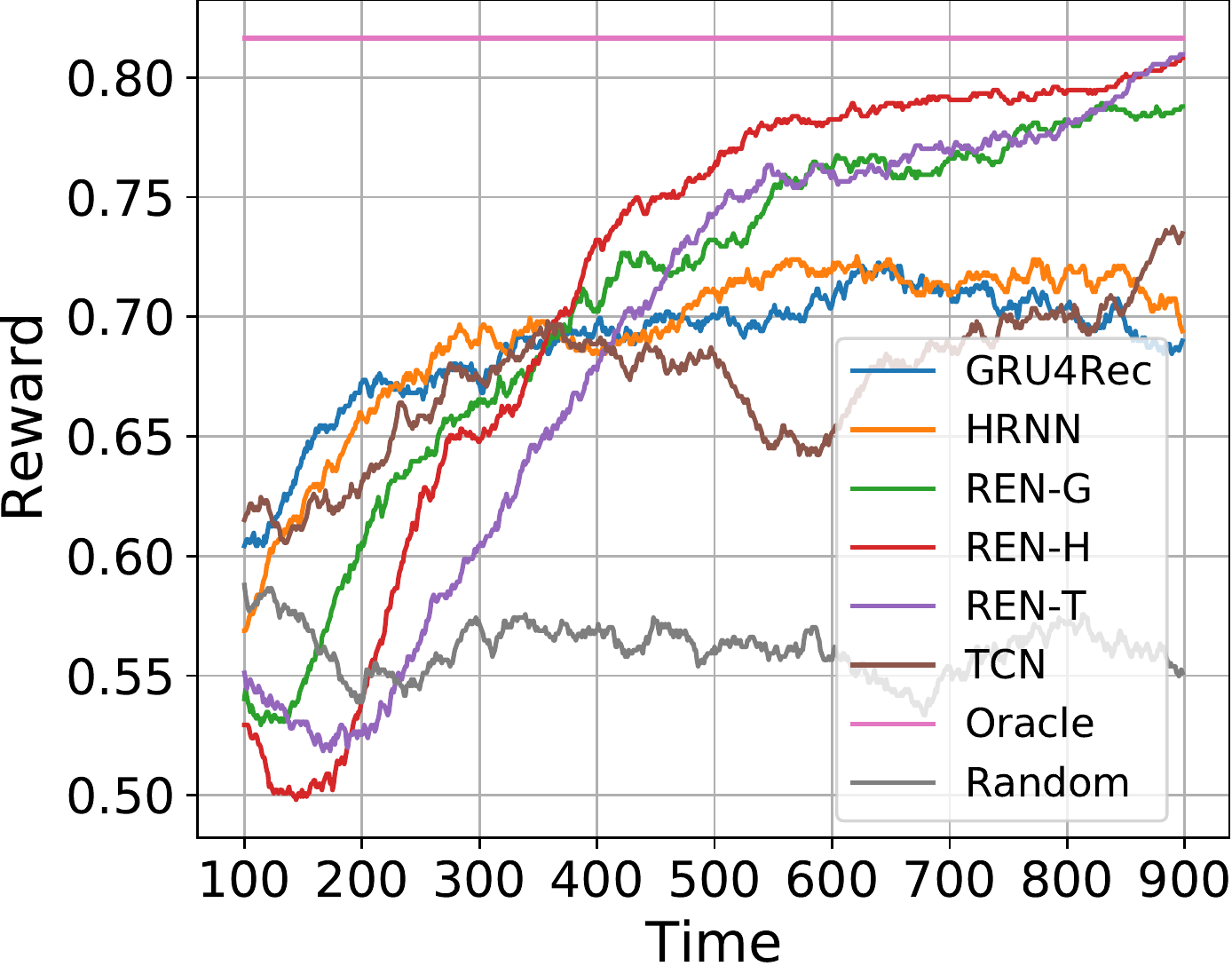}}
% \hspace{-0.05in}
\subfigure{
\includegraphics[width=0.325\textwidth]{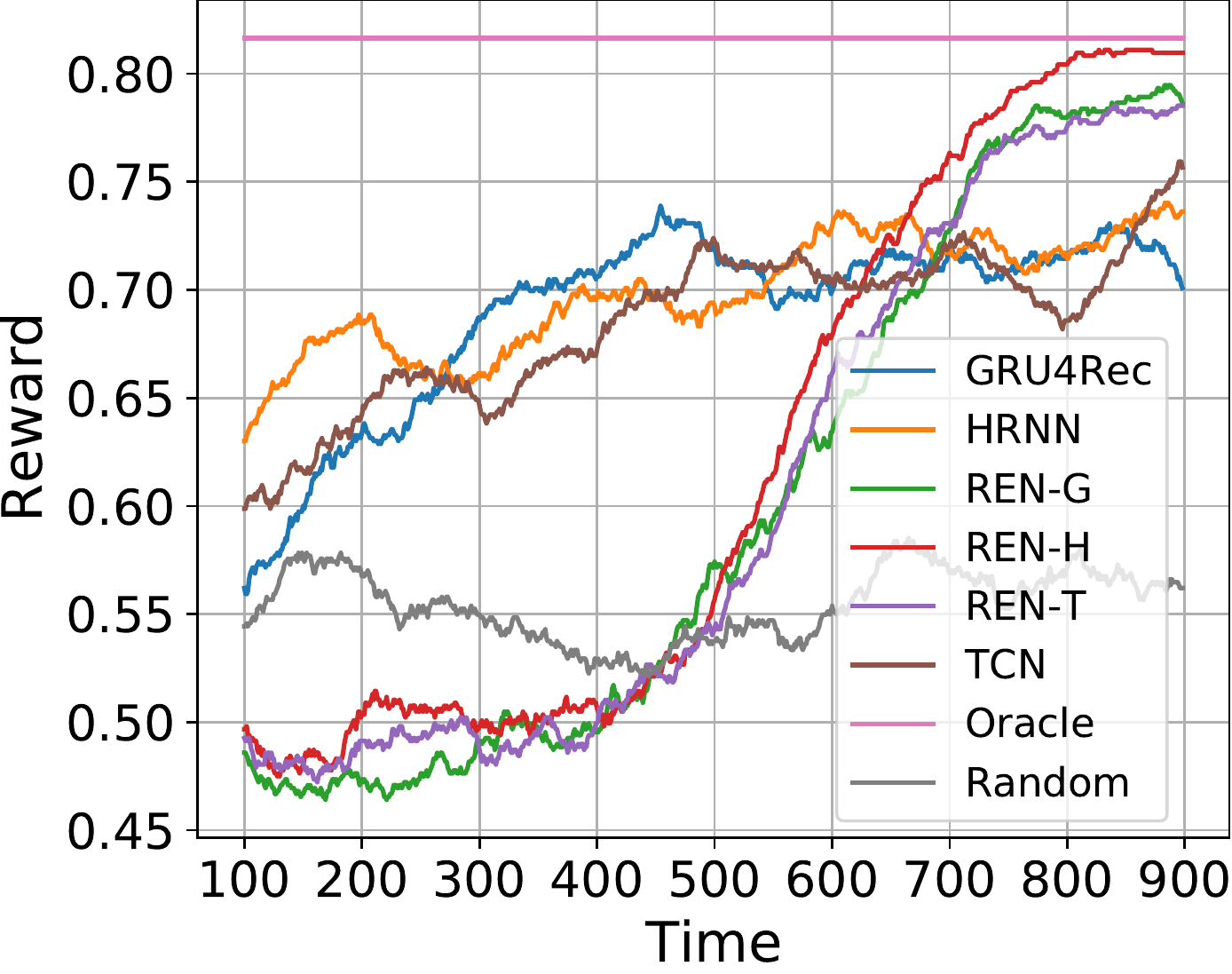}}
% \hspace{-0.05in}
\subfigure{
\includegraphics[width=0.325\textwidth]{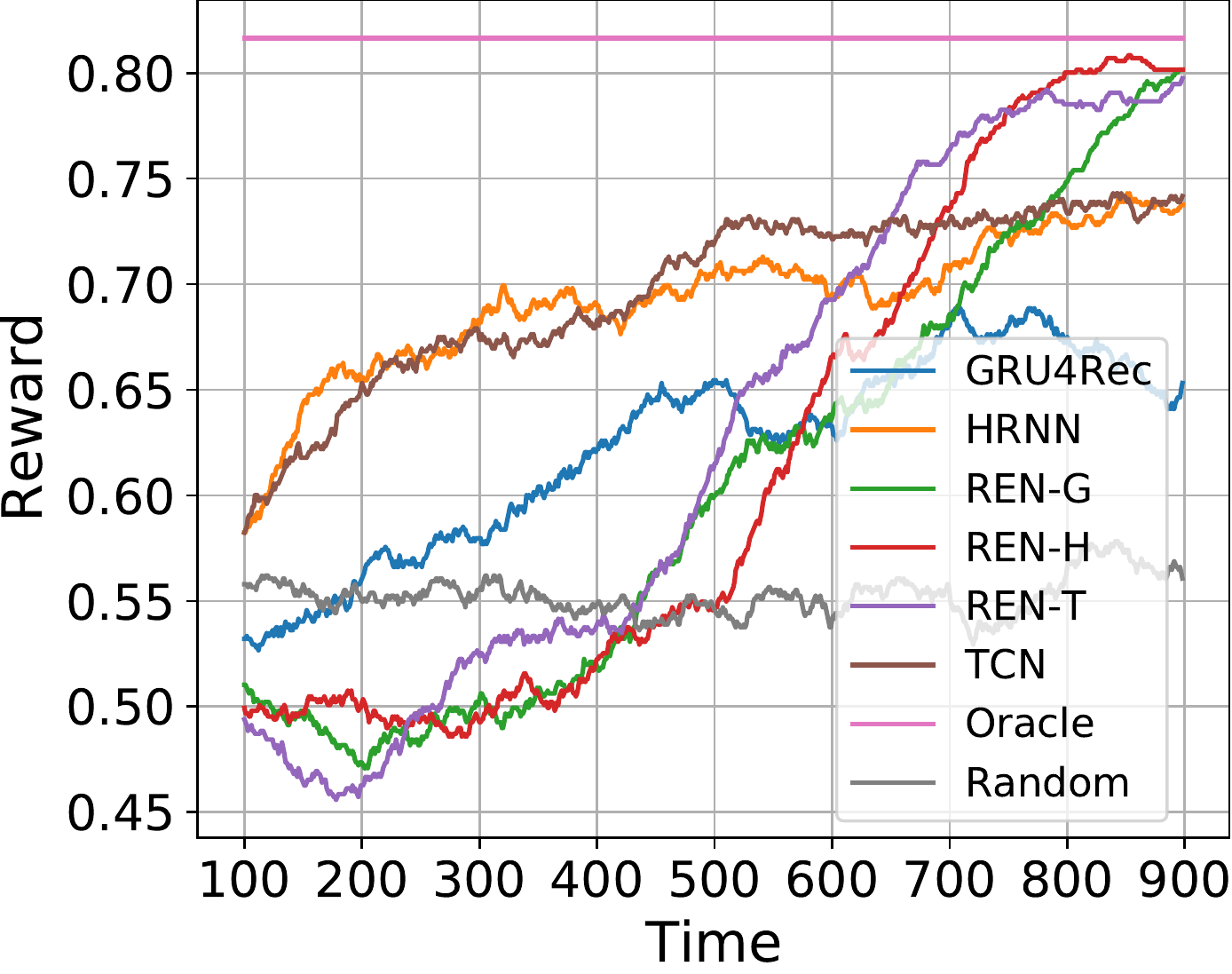}}
\end{center}
% \vskip -0.2in
\caption{Results for different methods in \emph{SYN-S} (\textbf{left} with 28 items), \emph{SYN-M} (\textbf{middle} with 280 items), and \emph{SYN-L} (\textbf{right} with 1400 items). One time step represents one interaction step, where in each interaction step the model recommends 3 items to the user and the user interacts with one of them. In all cases, REN models with diversity-based exploration lead to final convergence, whereas models without exploration get stuck at local optima. 
}
\label{fig:result_syn}
% \vskip -0.25in
\end{figure*}

% joint learning and exploration

\subsection{Simulated Experiments}\label{sec:simulated}
\textbf{Datasets.}
Following the setting described in \secref{sec:setup}, we start with three synthetic datasets, namely \emph{SYN-S}, \emph{SYN-M}, and \emph{SYN-L}, which allow complete control on the simulated environments. We assume $8$-dimensional latent vectors, which are unknown to the models, for each user and item, and use the inner product between user and item latent vectors as the reward. Specifically, for each latent user vector $\tha^*$, we randomly choose $3$ entries to set to $1/\sqrt{3}$ and set the rest to $0$, keeping $\|\tha^*\|_2 = 1$. We generate $C^8_2 = 28$ unique item latent vectors. Each item latent vector $\x_k^*$ has $2$ entries set to $1/\sqrt{2}$ and the other $6$ entries set to $0$ so that $\|\x_k^*\|_2=1$.

We assume $15$ users in our datasets. \emph{SYN-S} contains exactly $28$ items, while \emph{SYN-M} repeats each unique item latent vector for $10$ times, yielding $280$ items in total. Similarly, \emph{SYN-L} repeats for $50$ times, therefore yielding $1400$ items in total. The purpose of allowing different items to have identical latent vectors is to investigate REN's capability to explore in the compact latent space rather than the large item space. All users have a history length of $60$.
% HOG: last sentence can be deleted to save space

\textbf{Simulated Environments.}
With the generated latent vectors, the simulated environment runs as follows: At each time step $t$, the environment randomly chooses one user and feed the user's interaction history $\X_t$ (or $\D_t$) into the RNN recommender. The recommender then recommends the top $4$ items to the user. The user will select the item with the highest ground-truth reward ${\tha^*}^\top\x_k^*$, after which the recommender will collect the selected item with the reward and finetune the model. 
% In our experiments, all models update the parameters after every $5$ time steps.

\textbf{Results.}
\figref{fig:result_syn} shows the rewards over time for different methods. Results are averaged over $3$ runs and we plot the rolling average with a window size of $100$ to prevent clutter. As expected, conventional RNN-based recommenders saturate at around the $500$-th time step, while all REN variants successfully achieve nearly optimal rewards in the end. One interesting observation is that REN variants obtain rewards lower than the ``Random" baseline at the beginning, meaning that they are sacrificing immediate rewards to perform exploration in exchange for long-term rewards.

% \begin{wrapfigure}{R}{0.37\textwidth}
% \centering
% \vskip -0.0in
% \includegraphics[width=0.35\textwidth]{figures/ml-u1000-s120}
% \vskip -0.2cm
% % \captionsetup{font={scriptsize}}
% \caption{\label{fig:ml}Rewards over time on \emph{MovieLens-1M}. Models are finetuned every $10$ time steps (rounds).}
% \vskip -0.0cm
% \end{wrapfigure}

% \begin{wrapfigure}{R}{0.37\textwidth}
\begin{figure}[!tb]
\centering
% \vskip -0.4cm
\includegraphics[width=0.35\textwidth]{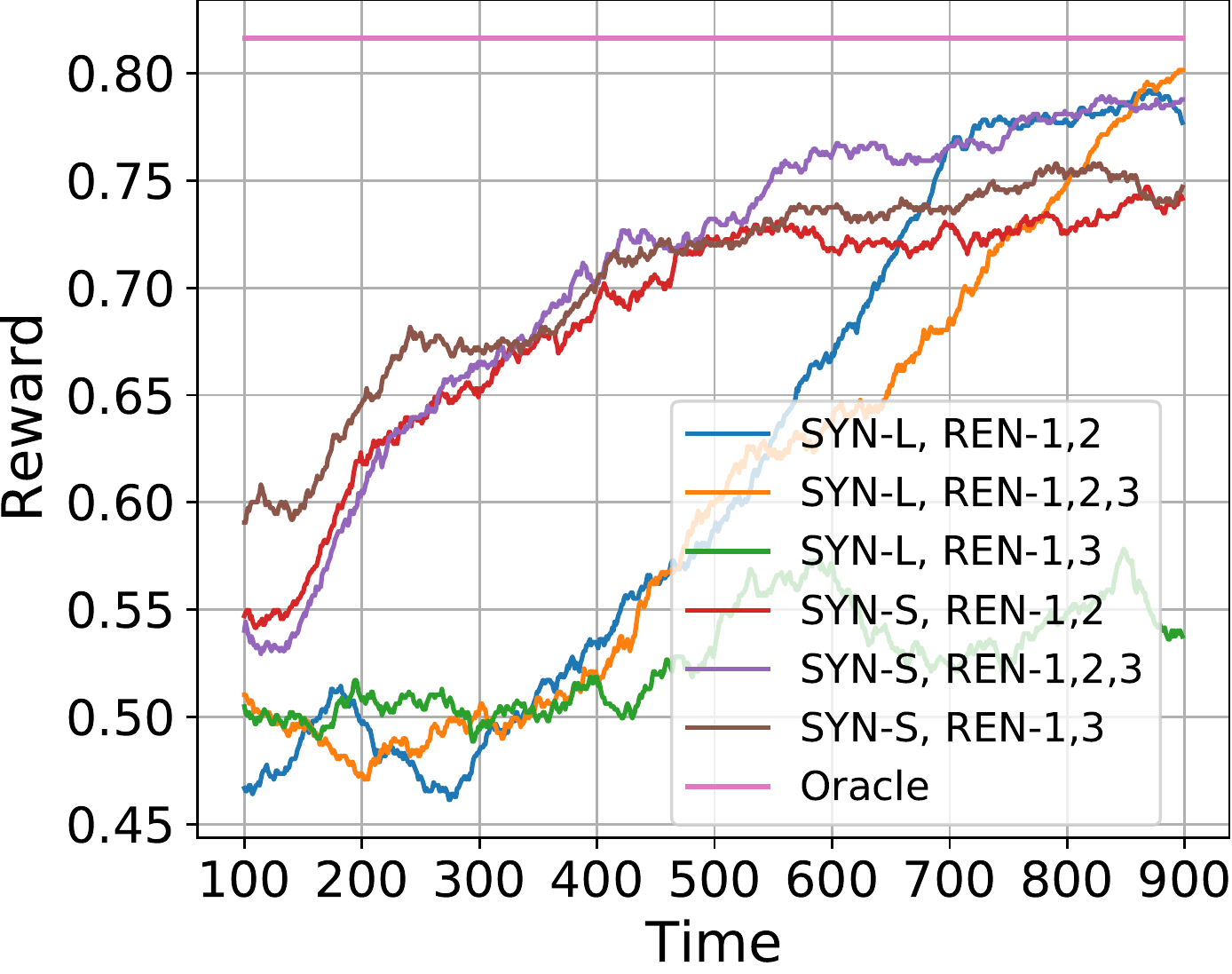}
% \vskip -0.2cm
% \captionsetup{font={scriptsize}}
\caption{\label{fig:ablation}Ablation study on different terms of REN. `REN-1,2,3' refers to the full `REN-G' model.}
% \vskip -0.9cm
\end{figure}
% \end{wrapfigure}

\textbf{Ablation Study.}
\figref{fig:ablation} shows the rewards over time for REN-G (i.e., REN-1,2,3), REN-1,2, and REN-1,3 in \emph{SYN-S} and \emph{SYN-L}. We observe that REN-1,2, with only the relevance (first) and diversity (second) terms of \eqnref{eq:rdn_w_uncertainty}, saturates prematurely in \emph{SYN-S}. On the other hand, the reward of REN-1,3, with only the relevance (first) and uncertainty (third) term, barely increases over time in \emph{SYN-L}. In contrast, the full REN-G works in both \emph{SYN-S} and \emph{SYN-L}. This is because without the uncertainty term, REN-1,2 fails to effectively choose items with uncertain embeddings to explore. REN-1,3 ignores the diversity in the latent space and tends to explore items that have rarely been recommended; such exploration directly in the item space only works when the item number is small, e.g., in \emph{SYN-S}.

% \textbf{Hyperparameters.} 
% For the base models GRU4Rec, TCN, and HRNN, we use identical network architectures and hyperparemeters whenever possible following~\cite{GRU4Rec,TCN,HRNN}. Each RNN consists of an encoding layer, a core RNN layer, and a decoding layer. We set the number of hidden neurons to $32$ for all models including REN variants (See the Supplement for more details on hyperparameters). 

%TODO: use this if there is a third figure.
\begin{figure*}[!tb]
\begin{center}
\vskip -0.1in
\hskip-0.1cm
%\framebox[4.0in]{$\;$}
%\includegraphics[height=5cm]{likeli1.eps}
\subfigure{
\includegraphics[height=3.8cm]{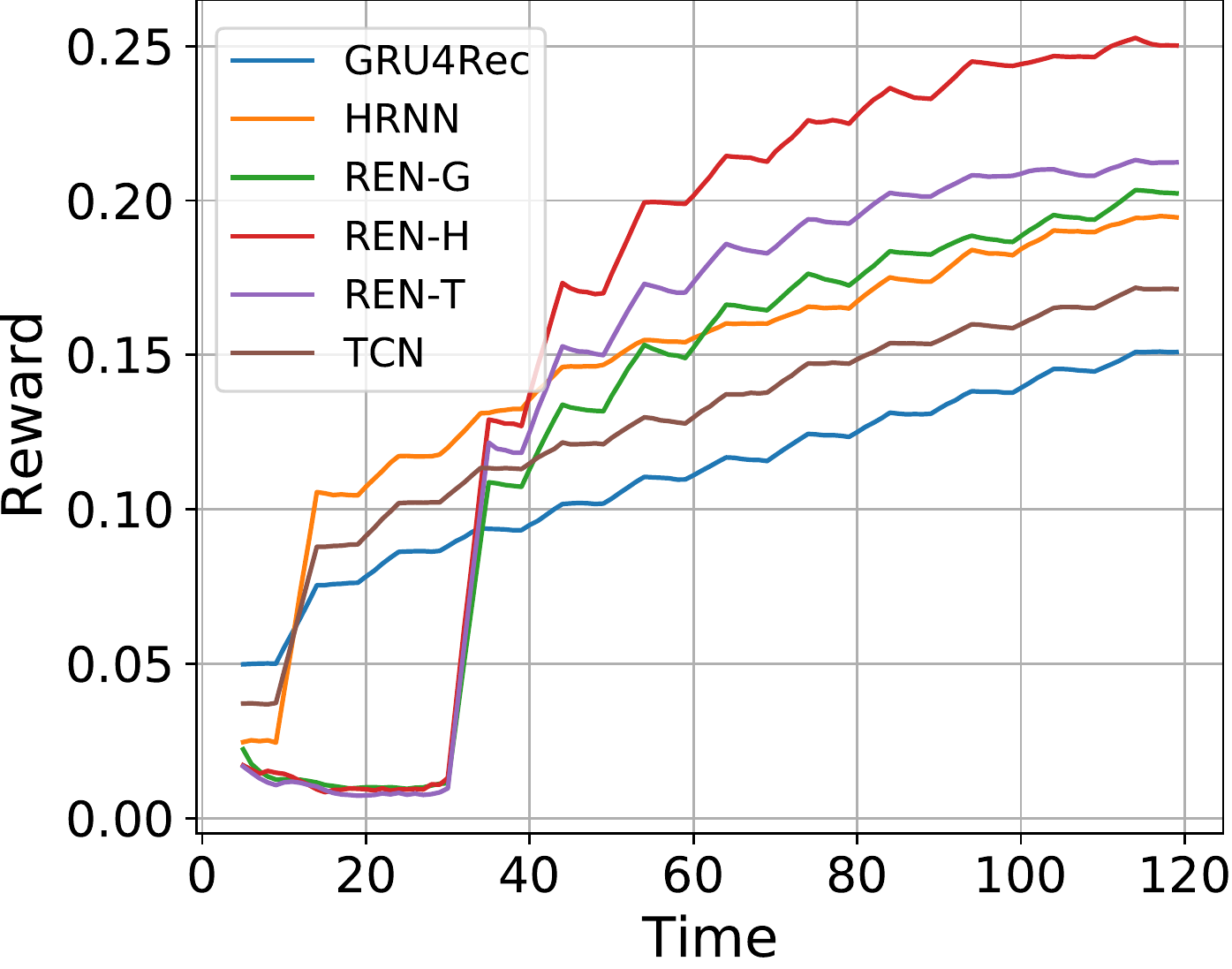}}
% \hspace{-0.05in}
% \hspace{0.25in}
\subfigure{
\includegraphics[height=3.8cm]{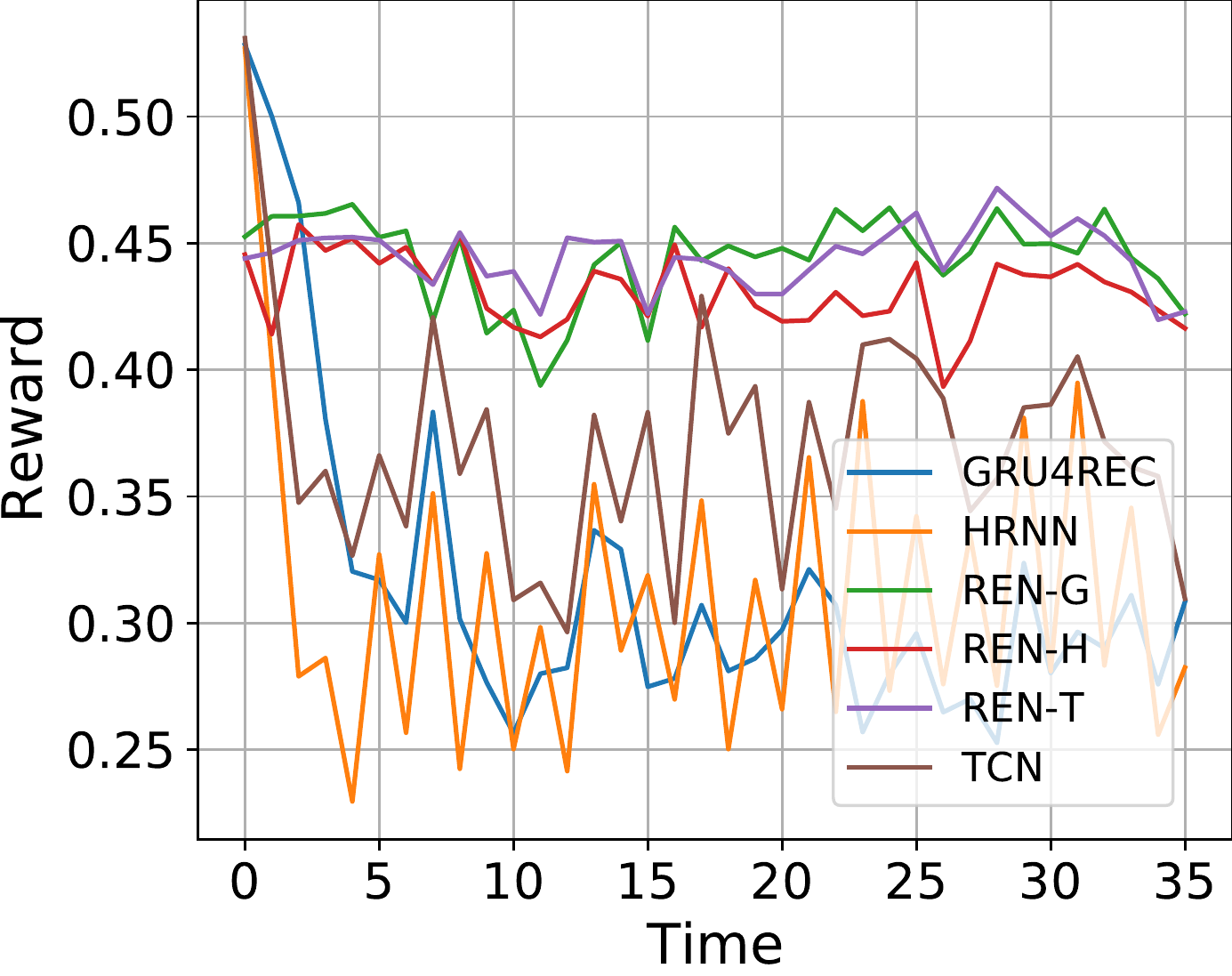}}
% \hspace{-0.05in}
\subfigure{
\includegraphics[height=3.8cm]{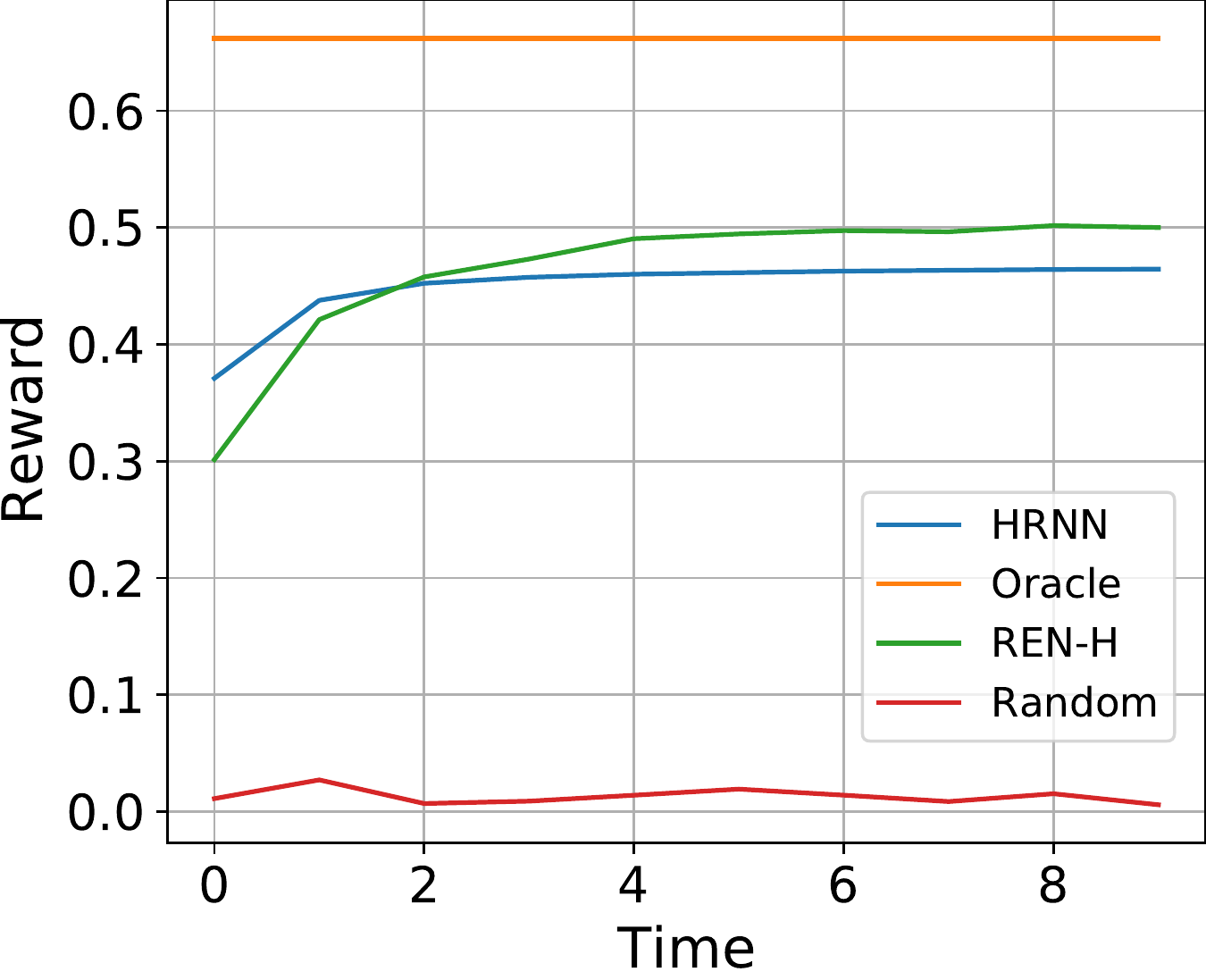}}
\end{center}
\vskip -0.2in
\caption{Rewards (precision@$10$, MRR, and recall@$100$, respectively) over time on \emph{MovieLens-1M} (\textbf{left}), \emph{Trivago} (\textbf{middle}), and \emph{Netflix} (\textbf{right}). One time step represents $10$ recommendations to a user, one hour of data, and $100$ recommendations to a user for \emph{MovieLens-1M}, \emph{Trivago}, and \emph{Netflix}, respectively.% For \emph{Trivago}, one time step represents one hour of data, while for \emph{Netflix} one time step represents $100$ recommendations to a user.
}
\label{fig:real}
\vskip -0.2cm
\end{figure*}

\begin{figure*}[!tb]
\begin{center}
%\framebox[4.0in]{$\;$}
%\includegraphics[height=5cm]{likeli1.eps}
\hskip-0.1cm
\subfigure{
\includegraphics[height=3.8cm]{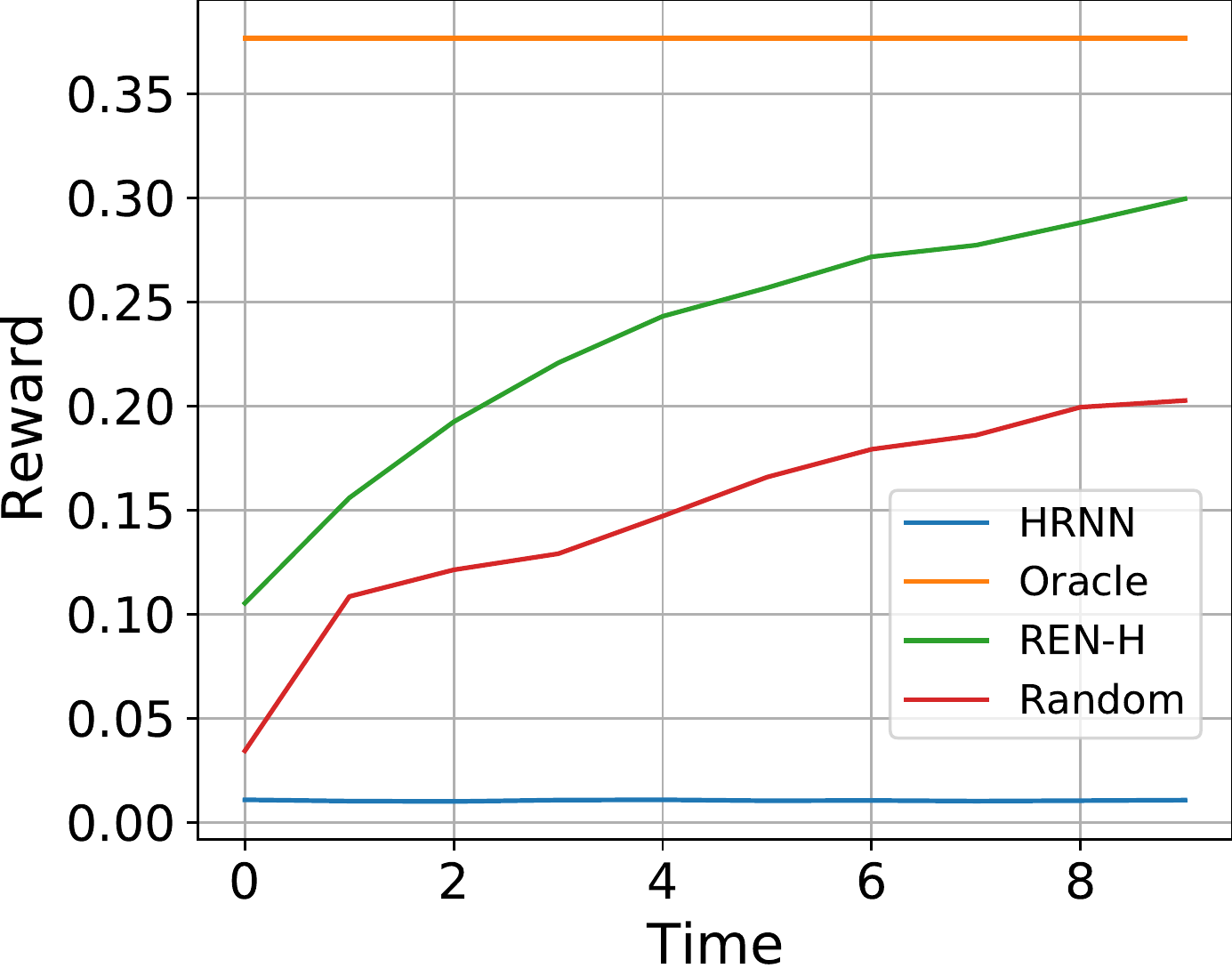}}
\subfigure{
\includegraphics[height=3.8cm]{figures/netflix-in-recall}}
\subfigure{
\includegraphics[height=3.8cm]{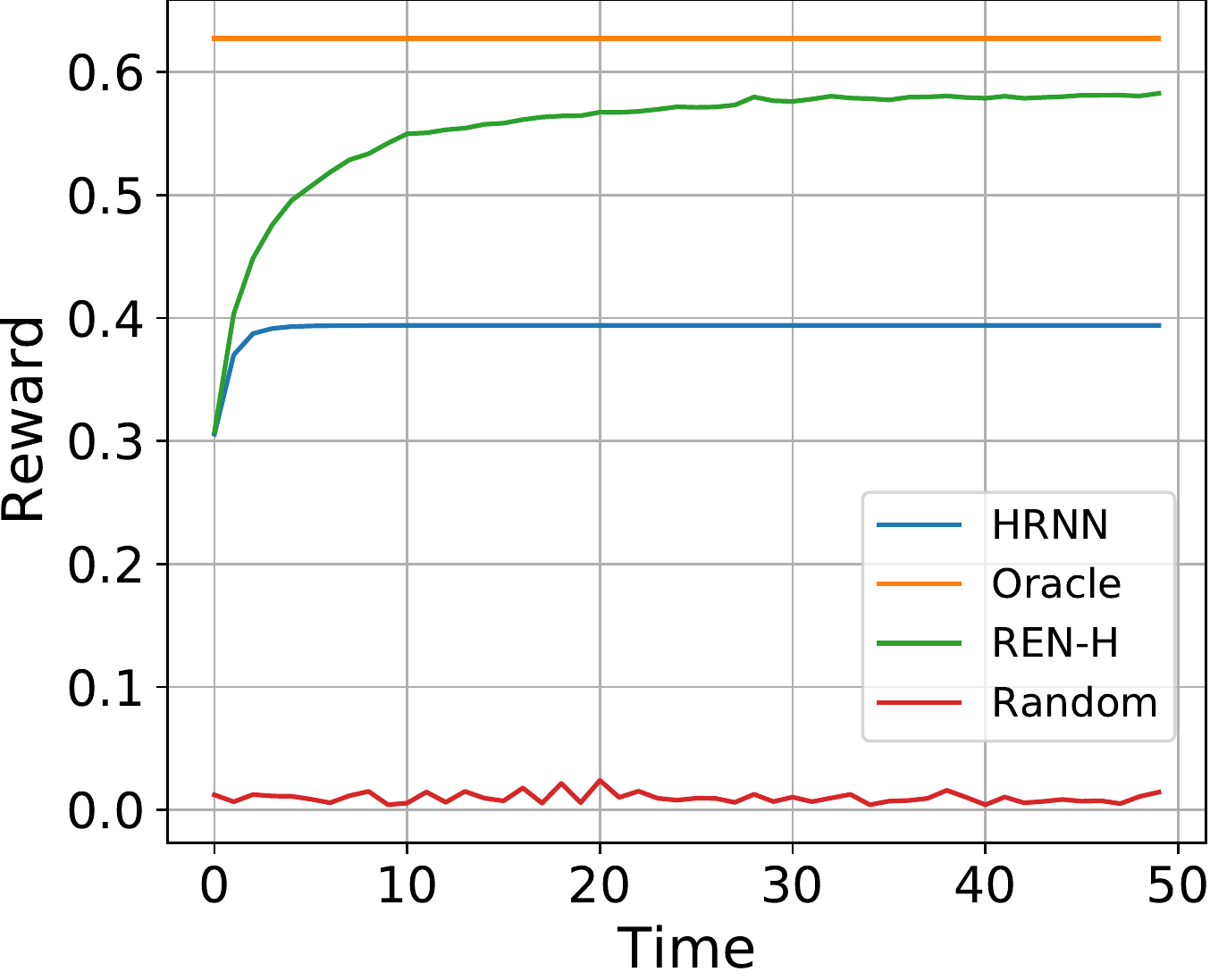}}
\end{center}
\vskip -0.2in
\caption{\label{fig:netflix_recall}Rewards over time on \emph{Netflix}. One time step represents $100$ recommendations to a user.
}
\vskip -0.6cm
\end{figure*}

\textbf{Hyperparameters.} 
For the base models GRU4Rec, TCN, and HRNN, we use identical network architectures and hyperparemeters whenever possible following~\cite{GRU4Rec,TCN,HRNN}. Each RNN consists of an encoding layer, a core RNN layer, and a decoding layer. We set the number of hidden neurons to $32$ for all models including REN variants.  
% \figref{fig:sensitivity} 
Fig. 1 in the Supplement shows the REN-G's performance for different $\lambda_d$ (note that we fix $\lambda_u = \sqrt{10}\lambda_d$) in \emph{SYN-S}, \emph{SYN-M}, and \emph{SYN-L}. We can observe stable REN performance across a wide range of $\lambda_d$. As expected, REN-G's performance is closer to GRU4Rec when $\lambda_d$ is small.

% \figref{fig:sensitivity} shows the REN-G's performance for different $\lambda_d$ (note that we fix $\lambda_u = \sqrt{10}\lambda_d$) in \emph{SYN-S}, \emph{SYN-M}, and \emph{SYN-L}. We can observe stable REN performance across a wide range of $\lambda_d$. As expected, REN-G's performance is closer to GRU4Rec when $\lambda_d$ is small.

\subsection{Real-World Experiments}\label{sec:real}
% TODO: add: item user cold start user cold start

\textbf{\emph{MovieLens-1M}.}
We use \emph{MovieLens-1M}~\citep{harper2016movielens} containing $3{,}900$ movies and $6{,}040$ users with an experiment setting similar to~\secref{sec:simulated}. Each user has $120$ interactions, and we follow the joint learning and exploration procedure described in~\secref{sec:setup} to evaluate all methods (more details in the Supplement). All models recommend $10$ items at each round for a chosen user, and the precision@$10$ is used as the reward. \figref{fig:real}(left) shows the rewards over time averaged over all $6{,}040$ users. As expected, REN variants with different base models are able to achieve higher long-term rewards compared to their non-REN counterparts.
% mention for more details please refer to supp
% filtered out users with short history and 

% \begin{wrapfigure}{R}{0.37\textwidth}
% \centering
% \vskip -0.1in
% \includegraphics[width=0.35\textwidth]{figures/trivago-with-initial-meta.pdf}
% \vskip -0.2cm
% % \captionsetup{font={scriptsize}}
% \caption{\label{fig:trivago}Rewards over time on \emph{Trivago}. One time step represents one hour of data.}
% \vskip -0.2cm
% \end{wrapfigure}

% To be discussed.
\textbf{\emph{Trivago}.}
We also evaluate the proposed methods on \emph{Trivago}\footnote{More details are available at \url{https://recsys.trivago.cloud/challenge/dataset/}.}, a hotel recommendation dataset with $730{,}803$ users, $926{,}457$ items, and $910{,}683$ interactions. We use a subset with $57{,}778$ users, $387{,}348$ items, and $108{,}713$ interactions and slice the data into $M=48$ one-hour time intervals for the online experiment (see the Supplement for details on data pre-processing). Different from \emph{MovieLens-1M}, \emph{Triavago} has impression data available: at each time step, besides which item is clicked by the user, we also know which $25$ items are being shown to the user. Such information makes the online evaluation more realistic, as we now know the ground-truth feedback if an arbitrary subset of the $25$ items are presented to the user. At each time step of the online experiments, all methods will choose $10$ items from the $25$ items to recommend the current user and collect the feedback for these $10$ items as data for finetuning. 
We pretrain the model using \emph{all $25$ items} from the first $13$ hours before starting the online evaluation. \figref{fig:real}(middle) shows the mean reciprocal rank (MRR), the official metric used in the RecSys Challenge, for different methods. As expected, the baseline RNN (e.g., GRU4Rec) suffers from a drastic drop in rewards because agents are allowed to recommend \emph{only $10$ items}, and they choose to focus only on relevance. This will inevitably ignores valuable items and harms the accuracy. In contrast, REN variants (e.g., REN-G) can effectively balance relevance and exploration for these $10$ recommended items at each time step, achieving higher long-term rewards. Interestingly, we also observe that REN variants have better stability in performance compared to RNN baselines.

% \begin{wrapfigure}{R}{0.37\textwidth}
% \centering
% \vskip -0.1in
% \includegraphics[width=0.35\textwidth]{figures/netflix}
% \vskip -0.2cm
% % \captionsetup{font={scriptsize}}
% \caption{\label{fig:coverage}Coverage over time on \emph{Netflix}. \hw{(place holder)}}
% \vskip -0.2cm
% \end{wrapfigure}

\textbf{\emph{Netflix}.} 
Finally, we also use \emph{Netflix}\footnote{\url{https://www.kaggle.com/netflix-inc/netflix-prize-data}} to evaluate how REN performs in the slate recommendation setting and without finetuning in each time step, i.e., skipping Step (3) in~\secref{sec:setup}. We pretrain REN on data from half the users and evaluate on the other half. At each time step, REN generates 100 mutually diversified items for one slate following \eqnref{eq:rdn_w_uncertainty}, with $p_{k,t}$ updated after every item generation. \figref{fig:real}(right) shows the recall@100 as the reward for different methods, demonstrating REN's promising exploration ability when no finetuning is allowed (more results in the Supplement). 
\figref{fig:netflix_recall}(left) shows similar trends with recall@100 as the reward on the same holdout item set.
This shows that the collected set contributes to building better user embedding models. \figref{fig:netflix_recall}(middle) shows that the additional exploration power comes without significant harms to the user's immediate rewards on the exploration set, where the recommendations are served.
In fact, we used a relatively large exploration coefficient, $\lambda_d=\lambda_u=0.005$, which starts to affect recommendation results on the sixth position.
By additional hyperparameter tuning, we realized that to achieve better rewards on the exploration set, we may choose smaller $\lambda_d=0.0007$ and $\lambda_u=0.0008$.
\figref{fig:netflix_recall}(right) shows significantly higher recalls close to the oracle performance, where all of the users' histories are known and used as inputs to predict the top-100 personalized recommendations.\footnote{The gap between oracle and 100\% recall lies in the model approximation errors.}
Note that, for fair presentation of the tuned results, we switched the exploration set and the holdout set and used a different test user group, consisting of 1543 users.
We believe that the tuned results are generalizable with new users and items, but we also realize that the \emph{Netflix} dataset still has a significant popularity bias and therefore we recommend using larger exploration coefficients with real online systems.
The inference cost is 175 milliseconds to pick top-100 items from 8000 evaluation items. It includes 100 sequential linear function solutions with 50 embedding dimensions, which is further improvable by selecting multiple items at a time in slate generation.

\section{Conclusion}
We propose the REN framework to balance relevance and exploration during recommendation. Our theoretical analysis and empirical results demonstrate the importance of considering uncertainty in the learned representations for effective exploration and improvement on long-term rewards. We provide an upper confidence bound on the estimated rewards along with its corresponding regret bound and show that REN can achieve the same rate-optimal sublinear regret even in the presence of uncertain representations. Future work could investigate the possibility of learned uncertainty in representations, extension to Thompson sampling, nonlinearity of the reward w.r.t. $\tha_t$, and applications beyond recommender systems, e.g., robotics and conversational agents.
%%%%%%%%%%%%%%%%%%%%%%%%%%%%%%%%%%%%%%%%%%%%%%%%%%%%%%%%%%%%%%%%%%%%%%%%%%%%%%%
%%%%%%%%%%%%%%%%%%%%%%%%%%%%%%%%%%%%%%%%%%%%%%%%%%%%%%%%%%%%%%%%%%%%%%%%%%%%%%%

\section{Acknowledgement}
The authors thank Tim Januschowski, Alex Smola, the AWS AI's Personalize Team and ML Forecast Team, as well as the reviewers/SPC/AC for the constructive comments to improve the paper. We are also grateful for the RecoBandits package provided by Bharathan Blaji, Saurabh Gupta, and Jing Wang to facilitate the simulated environments. HW is partially supported by NSF Grant IIS-2127918.

\bibliography{example_paper}

\appendix
\onecolumn

\section{Proofs in the Main Paper}
In this section, we provide the detailed proofs for the lemmas and main theorem in the paper. 
% With REN's connection to contextual bandits, we can prove that with proper $\lambda_d$ and $\lambda_u$, \eqnref{eq:rdn_w_uncertainty} is actually the upper confidence bound that leads to long term rewards with a rate-optimal regret bound.% Specifically, theory, confidence bound, regret bound

% Following the BaseLinUCB-SupLinUCB decomposition of LinUCB~\citep{chu2011contextual}, we divide the procedure of REN into ``BaseREN" (\algref{alg:baserdn}) and ``SupREN" (\algref{alg:suprdn}) stages correspondingly. Unlike the existing work which primarily considers the randomness from the reward, we take into consider the uncertainty resulted from the context. In what follows, we first provide the high probability bound for the BaseREN algorithm with uncertain embeddings (context), and derive an upper bound for the regret. As mentioned in \secref{sec:notation}, for the online setting where the model updates at every time step $t$, $\x_k$ also changes over time. Therefore in this section we use $\x_{t,k}$, $\muu_{t,k}$, $\Si_{t,k}$, and $\si_{t,k}$ in place of $\x_k$, $\muu_{k}$, $\Si_{k}$, and $\si_{k}$ from \secref{sec:rdn} to be rigorous.

\begin{assup}
Assume there exists an optimal $\tha^*$, with $\|\tha^*\| \leq 1$ and $\x_{t,k}^*$ such that $\E[r_{t,k}] = {\x_{t,k}^*}^\top \tha^*$.  Further assume that there is an effective distribution $\NM(\muu_{t,k}, \Si_{t,k})$ such that ${\x}_{t,k}^* \sim \NM(\muu_{t,k}, \Si_{t,k})$ where $\Si_{t,k} = \textbf{diag}(\si_{t,k}^2)$. Thus, the true underlying context is unavailable, but we are aided with the knowledge that it is generated with a multivariate normal whose parameters are known. 
\end{assup}
% removed the word `posterior' to avoid attacks from Bayesian researchers

% \bernie{make the notation consistent, $\bm{x}$ for vector, $\mathbf{D}$ for matrix, $s$ for scala.}
% \hw{HW: done:) let me know if I missed anything}

% \hw{HW: Currently there is some disconnect between $\tha$ in LinUCB, which is updated online as a linear model, and $\tha$ in REN, which is from REN. Need to clarity this somewhere.}

\begin{algorithm}\label{alg:baserdn}
\caption{BaseREN: Basic REN Inference at Step $t$}
\SetAlgoLined
\textbf{Input:} $\alpha$, $\Psi_t\subseteq \{1,2,\dots,t-1\}$.\\
Obtain item embeddings from REN: $\muu_{\tau,k_\tau} \gets f_e(\e_{\tau,k_\tau})$ for all $\tau\in\Psi_t$.\\
Obtain the current user embedding from REN: $\tha_t\gets R(\D_t)$.\\
$\A_t \gets \I_d + \sum_{\tau\in \Psi_t}\muu_{\tau,k_\tau}^\top \muu_{\tau,k_\tau}$.\\
Obtain candidate items' embeddings from REN: $\muu_{t,k}\gets f_e(\e_{t,k})$, where $k\in [K]$.\\
Obtain candidate items' uncertainty estimates $\si_{t,k}$, where $k\in [K]$.\\
\For{$a\in [K]$}{
$w_{t,k} \gets (\alpha+1)s_{t,k} + (4\sqrt{d} + 2\sqrt{\ln \frac{TK}{\delta}})\|\si_{t,k}\|_{\infty}$.\label{algl:width}\\
$\widehat{r}_{t,k} \gets \tha_t^\top \muu_{t,k}$.
} % end of for
Recommend item $k\gets \argmax_k \widehat{r}_{t,k} + w_{t,k}$.
\end{algorithm}
\begin{algorithm}\label{alg:suprdn}
\caption{SupREN}
\SetAlgoLined
\textbf{Input:} Number of rounds $T$.\\
$S \gets \ln T$ and $\Psi_t^{(s)}\gets \emptyset$ for all $s\in[T]$.\\
\For{$t=1,2,\dots,T$}{
$s\gets 1$ and $\hat{A}_1\gets [K]$.\\
\Repeat{an item $k_t$ is found}{
    Use BaseREN with $\Psi_t^{(s)}$ to calculate the width, $w_{t,k}^{(s)}$, and the upper confidence bound, $\hat{r}^{(s)}_{t,k} + w_{t,k}^{(s)}$, for all $k\in\hat{A}_s$.\\

    \uIf{$w_{t,k}^{(s)}\leq \frac{1}{\sqrt{T}}$ for all $k\in \hat{A}_s$}{
        Choose  $k_t=\argmax_{k\in\hat{A}_s}(\widehat{r}_{t,k}^{(s)}+w_{t,k}^{(s)})$ and update: $\Psi_{t+1}^{(s')}\gets \Psi_t^{(s')}$ for all $s'\in [S]$. \label{algl:step2c}
    }
    \uElseIf{$w_{t,k}^{(s)} \leq 2^{-s}$ for all $k\in\hat{A}_s$}{
        $\hat{A}_{s+1} \gets \{ k\in\hat{A}_s | \widehat{r}_{t,k}^{(s)} + w_{t,k}^{(s)} \geq \max_{k'\in \hat{A}_s}(\hat{r}_{t,k'}^{(s)}+w_{t,k'}^{(s)}) - 2^{1-s}\}$, $s\gets s+1$.
    }
    \Else{
        Choose $k_t\in\hat{A}_s$ such that $w_{t,k_t}^{(s)}>2^{-s}$ and update: 
        $\Psi_{t+1}^{(s)} \gets \Psi_t^{(s)}\cup\{t\}$, $\Psi_{t+1}^{(s')} \gets \Psi_t^{(s')}$ for $s'\neq s$. \label{algl:step2b}
    }    
} % end of repeat
Update the REN model $R(\cdot)$ and $f_e(\cdot)$ using collected user feedbacks.
} % end of for

\end{algorithm}

\subsection{Upper Confidence Bound for Uncertain Embeddings}\label{sec:confidence_bound}

For simplicity we follow the notation from~\cite{chu2011contextual} and denote the item embedding (context) as $\x_{t,k}$, where $t$ indexes the rounds and $k$ indexes the items. We define:
\begin{align*}
s_{t,k} &= \sqrt{\muu_{t,k}^\top \A_{t}^{-1} \muu_{t,k}} \in \mathbb{R}_+, \;\;
\D_t = [\muu_{\tau,k_{\tau}}]_{\tau \in \Psi_t} \in \mathbb{R}^{|\Psi_t|\times d}, \\
\y_t &= [r_{\tau,k_{\tau}}]_{\tau \in \Psi_t} \in \mathbb{R}^{|\Psi_t|\times 1}, \;\;
\A_t = \I_d + \D_t^\top \D_t, \\
\b_t &= \D_t^\top \y_t, \;\;\;\;\;\;\;\;\;\;\;\;\;\;\;\;\;\;\;\;\;\;\;\;
% \hat{\tha} &= R(\X_t) \\
\widehat{r}_{t,k} = \muu_{t,k}^\top \hat{\tha} = \muu_{t,k}^\top \A_t^{-1} \b_t,
\end{align*}
where $\y_t$ is the collected user feedback. \lemref{lem:simple_cb} below shows that with $\lambda_d=1+\alpha=1+\sqrt{\frac{1}{2}\ln \frac{2TK}{\delta}}$ and $\lambda_u=4\sqrt{d} + 2\sqrt{\ln \frac{TK}{\delta}}$, the main equation in the paper is the upper confidence bound with high probability, meaning that it upper bounds the true reward with high probability, which makes it a reasonable score for recommendations. % and $\hat{\tha}$ is the estimated user embedding from REN.
% \begin{align*}
% s_{t,k} &= \sqrt{\muu_{t,k}^\top \A_{t}^{-1} \muu_{t,k}} \in \mathbb{R}_+ \\
% \D_t &= [\muu_{\tau,k_{\tau}}]_{\tau \in \Psi_t} \in \mathbb{R}^{|\Psi_t|\times d} \\
% \y_t &= [r_{\tau,k_{\tau}}]_{\tau \in \Psi_t} \in \mathbb{R}^{|\Psi_t|\times 1} \\
% \A_t &= \I_d + \D_t^\top \D_t \\
% \b_t &= \D_t^\top \y_t \\
% % \hat{\tha} &= R(\X_t) \\
% \widehat{r}_{t,k} &= \muu_{t,k}^\top \hat{\tha} = \muu_{t,k}^\top \A_t^{-1} \b_t
% \end{align*}
% \begin{align*}
% s_{t,k} = \sqrt{x_{t,k}^\top A_{t}^{-1} x_{t,k}} \in \mathbb{R}_+
% \end{align*}

\begin{lemma}[\textbf{Confidence Bound}]\label{lem:simple_cb}
With probability at least $1 - 2 \delta / T$, we have for all $k\in[K]$ that
\begin{align*}
|\widehat{r}_{t,k} - {\x_{t,k}^*}^\top \tha^*| \leq (\alpha + 1)s_{t,k} + (4\sqrt{d} + 2\sqrt{\ln \frac{TK}{\delta}})\|\si_{t,k}\|_{\infty},
\end{align*}
where $\|\si_{t,k}\|_{\infty} = \max_i |\si_{t,k}^{(i)}|$ is the $L_{\infty}$ norm.
\end{lemma}

\begin{proof}
Using the notation defined above, we have
\begingroup\makeatletter\def\f@size{8}\check@mathfonts
\begin{align}
&\quad\quad |\widehat{r}_{t,k} - {\x_{t,k}^*}^\top \tha^*| \nonumber\\
&\quad = |\muu_{t,k}^\top \A_t^{-1} b_t - {\x_{t,k}^*}^\top \A_t^{-1}(\I_d + \D_t^\top \D_t)\tha^* | \nonumber\\
&\quad = |\muu_{t,k}^\top \A_t^{-1} \D_t^\top \y_t - {\x_{t,k}^*}^\top \A_t^{-1}(\tha^* + \D_t^\top \D_t \tha^*) | \nonumber\\
&\quad =|\muu_{t,k}^\top \A_t^{-1} \D_t^\top \y_t - {\x_{t,k}^*}^\top \A_t^{-1} \D_t^\top \D_t \tha^* - {\x_{t,k}^*}^\top \A_t^{-1}\tha^* |\nonumber\\
&\quad =|(\muu_{t,k}^\top \A_t^{-1} \D_t^\top \y_t - \muu_{t,k}^\top \A_t^{-1} \D_t^\top \D_t \tha^*) + \muu_{t,k}^\top \A_t^{-1} \D_t^\top \D_t \tha^* - {\x_{t,k}^*}^\top \A_t^{-1} \D_t^\top \D_t \tha^* - {\x_{t,k}^*}^\top \A_t^{-1}\tha^* | \nonumber\\
&\quad = |\muu_{t,k}^\top \A_t^{-1} \D_t^\top (\y_t -  \D_t \tha^*) + (\muu_{t,k} - {\x_{t,k}^*})^\top \A_t^{-1} \D_t^\top \D_t \tha^* - {\x_{t,k}^*}^\top \A_t^{-1}\tha^* | \nonumber\\
&\quad = |\muu_{t,k}^\top \A_t^{-1} \D_t^\top (\y_t -  \D_t \tha^*) - (\Si_{t,k}^{1/2}\ep)^\top \A_t^{-1} \D_t^\top \D_t \tha^* - (\muu_{t,k} + \Si_{t,k}^{1/2}\ep)^\top \A_t^{-1}\tha^* | \nonumber\\
&\quad = |\muu_{t,k}^\top \A_t^{-1} \D_t^\top (\y_t -  \D_t \tha^*) - (\Si_{t,k}^{1/2}\ep)^\top\tha^* - (\muu_{t,k})^\top \A_t^{-1}\tha^* | \label{eq:merge_ad}\\
&\quad\leq | \muu_{t,k}^\top \A_t^{-1} \D_t^\top (\y_t -  \D_t \tha^*) | + \|\Si_{t,k}^{1/2}\ep\| + s_{t,k}. \label{eq:eliminate_theta}
\end{align}
\endgroup

To see \eqnref{eq:merge_ad} is true, note that $\A_t^{-1}\D_t^\top\D_t^\top + \A_t^{-1} = \A_t^{-1}(\D_t^\top\D_t^\top +\I_d) = \I_d$. To see \eqnref{eq:eliminate_theta} is true, note that since $\|\tha^*\| \leq 1$, we have $|(\Si_{t,k}^{1/2}\ep)^\top\tha^*| \leq \|\Si_{t,k}^{1/2}\ep\|$. Similarly for the last term in \eqnref{eq:eliminate_theta}, observe that
\begin{align}
& \| \A_t^{-1} \muu_{t,k}\| \nonumber \\
=& \sqrt{\muu_{t,k}^\top \A_t^{-1} \I_d \A_t^{-1} \muu_{t,k}}  \nonumber \\
\leq& \sqrt{\muu_{t,k}^\top \A_t^{-1} (\I_d + \D_t^\top\D_t) \A_t^{-1} \muu_{t,k}} \nonumber \\
=& \sqrt{\muu_{t,k}^\top \A_t^{-1} \muu_{t,k}} \nonumber \\
=& s_{t,k}. \label{eq:tri3}
\end{align}

% With $\|\tha^*\| \leq 1$, we have
% \begin{align}
% |\widehat{r}_{t,k} - {\x_{t,k}^*}^\top \tha^*| \leq |\muu_{t,k}^\top \A_t^{-1} \D_t^\top \y_t - {\x_{t,k}^*}^\top \A_t^{-1} \D_t^\top \D_t \tha^*| + \| \A_t^{-1}{\x_{t,k}^*}\|.
% \label{eq:tri1}
% \end{align} 
% Since ${\x}_{t,k}^* \sim \NM(\muu_{t,k}, \Si_{t,k})$, we have ${\x}_{t,k}^* =  \Si_{t,k}^{1/2}\ep + \muu_{t,k}$, where $\ep \sim \NM(\0, \I_d)$. Therefore, the first term on RHS of \eqnref{eq:tri1}
% \begin{align}
% &|\muu_{t,k}^\top \A_t^{-1} \D_t^\top \y_t - {\x_{t,k}^*}^\top \A_t^{-1} \D_t^\top \D_t \tha^*| \nonumber \\
% =& |\muu_{t,k}^\top \A_t^{-1} \D_t^\top \y_t - (\Si_{t,k}^{1/2}\ep + \muu_{t,k})^\top \A_t^{-1} \D_t^\top \D_t \tha^*| \nonumber \\
% =& |\muu_{t,k}^\top \A_t^{-1} \D_t^\top \y_t -  \muu_{t,k}^\top \A_t^{-1} \D_t^\top \D_t \tha^* - \ep ^\top \Si_{t,k}^{1/2} \A_t^{-1} \D_t^\top \D_t \tha^*| \nonumber \\
% \leq& |\muu_{t,k}^\top \A_t^{-1} \D_t^\top (\y_t -  \D_t \tha^*) | + \|\ep ^\top \Si_{t,k}^{1/2} A_t^{-1} \D_t^\top \D_t \|, \label{eq:tri2}
% \end{align}
% where \eqnref{eq:tri2} is true because $\|\tha^*\| \leq 1$. 
For the first term in \eqnref{eq:eliminate_theta}, since $\E[\y_t - \D_t\tha^*] = 0$, and $\muu_{t,k}^\top \A_t^{-1} \D_t^\top \y_t$ is a random variable bounded by $\|\D_t \A_t^{-1} \muu_{t,k}\|$, by Azuma-Hoeffding inequality, we have
\begin{align}
&\Pr(|\muu_{t,k}^\top \A_t^{-1} \D_t^\top (\y_t -  \D_t \tha^*)| > \alpha s_{t,k}) \nonumber \\
\leq& 2\exp \left( -\frac{2\alpha^2 s_{t,k}^2}{\|\D_t \A_t^{-1} \muu_{t,k}\|^2} \right) \nonumber \\
\leq& 2\exp(-2\alpha^2) \label{eq:alpha} \\
=& \frac{\delta}{TK} \label{eq:union1},
\end{align}
where \eqnref{eq:alpha} is due to
\begin{align*}
s_{t,k}^2 &= \muu_{t,k}^\top \A_t^{-1} \muu_{t,k} \\
&= \muu_{t,k}^\top  \A^{-1}(\I_d + \D_t^\top\D_t) \A^{-1} \muu_{t,k} \\
&\geq \muu_{t,k}^\top  \A^{-1} \D_t^\top\D_t \A^{-1} \muu_{t,k} \\
&= \|\D_t \A_t^{-1} \muu_{t,k}\|^2.
\end{align*}

For the second term of \eqnref{eq:eliminate_theta}, $\|\ep ^\top \Si_{t,k}^{1/2} \|$, since $\ep^\top \Si_{t,k}^{1/2} \sim \NM(\0, \Si_{t,k})$, we can guarantee that with probability at most $\frac{\delta}{TK}$,
\begin{align}
\|\ep ^\top \Si_{t,k}^{1/2} \| 
> 2\sqrt{\lambda_{max}(\Si_{t,k})} (2\sqrt{d} + \sqrt{\ln \frac{TK}{\delta}}), \label{eq:union2}
\end{align}
where $\lambda_{max}(\Si_{t,k}) = \|\Si_{t,k}\|_{op}$ is the operator norm of the matrix $\Si_{t,k}$ corresponding to the $L_2$ vector norm.

Combining \eqnref{eq:eliminate_theta}, \eqnref{eq:union1}, and \eqnref{eq:union2}, with a union bound, we have that with probability at least $1 - \frac{2\delta}{T}$, for all actions $a \in [K]$,
\begin{align*}
|\widehat{r}_{t,k} - {\x_{t,k}^*}^\top \tha^*| 
&\leq (\alpha + 1)s_{t,k} + (4\sqrt{d} + 2\sqrt{\ln \frac{TK}{\delta}})\sqrt{\lambda_{max}(\Si_{t,k})}, \\
&= (\alpha + 1)s_{t,k} + (4\sqrt{d} + 2\sqrt{\ln \frac{TK}{\delta}})\|\si_{t,k}\|_{\infty},
\end{align*}
\end{proof}

\subsection{Regret Bound}
\lemref{lem:simple_cb} above provides a reasonable estimate of the reward's upper bound at time $t$. Based on this estimate, one natural next step is to analyze the regret after all $T$ rounds. Formally, we define the regret of the algorithm after $T$ rounds as
\begin{align}\label{eq:regret}
B(T) = \sum_{t=1}^T r_{t,k_t^*} - \sum_{t=1}^T r_{t,k_t},
\end{align}
where $k_t^*$ is the optimal item (action) $k$ at round $t$ that maximizes $\E[r_{t,k}] = \x_{t,k}^T\tha^*$, and $k_t$ is the action chose by the algorithm at round $t$. In a similar fashion as in~\cite{chu2011contextual}, SupREN calls BaseREN as a sub-routine. In this subsection, we derive the regret bound for SupREN with uncertain item embeddings.

\begin{lemma}[Azuma--Hoeffding Inequality]\label{lem:azuma}
Let $X_1, \dots, X_m$ be random variables with $|X_\tau|\leq a_\tau$ for some $a_1,\dots,a_m>0$. Then we have
\begin{align*}
\Pr(|\sum_{\tau=1}^m X_\tau - \sum_{\tau=1}^m \E[X_\tau|X_1,\dots,X_{\tau-1}]|\geq B)\leq 2\exp\left(-\frac{B^2}{2\sum_{\tau=1}^m a^2_\tau}\right).
\end{align*}
\end{lemma}

\begin{lemma}\label{lem:regret_to_width}
With probability $1-2\delta S$, for any $t\in [T]$ and any $s\in [S]$:
\begin{enumerate}
\item $|\widehat{r}_{t,k} - \E[r_{t,k}]|\leq w_{t,k}$ for any $k\in[K]$,
\item $k_t^*\in \hat{A}_s$, and
\item $\E[r_{t,k^*_t}] - \E[r_{t,k}] \leq 2^{(3-s)}$ for any $k\in\hat{A}_s$.
\end{enumerate}
\end{lemma}
\begin{proof}
The proof is a simple modification of that in \cite{LinRel} (Lemma 15) to accommodate modification in \lemref{lem:simple_cb}.
\end{proof}

\begin{lemma}\label{lem:s_to_sqrt}%[\citealt{chu2011contextual}, Lemma 6]
In BaseREN, we have
\begin{align*}
(1+\alpha)\sum_{t\in\Psi_{T+1}} s_{t,k_t} \leq 5 \cdot (1+\alpha^2) \sqrt{d|\Psi_{T+1}|}.
\end{align*}
\end{lemma}
\begin{proof}
This is a direct result of Lemma 3 and Lemma 6 in \cite{chu2011contextual} as well as Lemma 16 in \cite{LinRel}.
\end{proof}

\begin{lemma}\label{lem:sigma_to_sqrt}
Assuming $\|\si_{1,k}\|_{\infty} = 1$ and $\|\si_{t,k}\|_{\infty} \leq \frac{1}{\sqrt{t}}$ for any $k$ and $t$, then for any $k$,
\begin{align*}
\sum_{t\in\Psi_{T+1}} \|\si_{t,k}\|_{\infty} \leq \sqrt{|\Psi_{T+1}|}
\end{align*}
\end{lemma}
\begin{proof}
Since the function $f(t) = \frac{1}{\sqrt{t}}$ is convex when $t>0$, we have 
\begin{align*}
\sum_{t=1}^{|\Psi_{T+1}|}\frac{1}{\sqrt{t}}\leq \int_0^{|\Psi_{T+1}|} \frac{1}{\sqrt{t}} =\left. \sqrt{t} \right\vert_0^{|\Psi_{T+1}|} = \sqrt{|\Psi_{T+1}|}
\end{align*}
\end{proof}
% Essentially \lemref{lem:regret_to_width} links the regret $B(T)$ to the width of the confidence bound $w_{t,k}$ (Line~\ref{algl:width} of \algref{alg:baserdn} or the last two terms of \eqnref{eq:rdn_w_uncertainty}). \lemref{lem:s_to_sqrt} and \lemref{lem:sigma_to_sqrt} then connect $w_{t,k}$ to $\sqrt{|\Psi_{T+1}|}\leq \sqrt{T}$, which is sublinear in $T$; this is the key to achieve a sublinear regret bound.

% Interestingly, \lemref{lem:sigma_to_sqrt} states that the uncertainty only needs to decrease at the rate $\frac{1}{\sqrt{t}}$, which is consistent with our choice of $\mathbf{diag}(\si_{k}) = 1/\sqrt{n_{k}}\;\I_d$ in~\secref{sec:rdn_final}, where $n_{k}$ is item $k$'s total number of impressions for all users. 

% As the last step, \lemref{lem:psi_to_sqrt} and \thmref{thm:regret_bound} below build on all lemmas above to derive the final sublinear regret bound. 

\begin{lemma}\label{lem:psi_to_sqrt}
For all $s\in [S]$,
\begin{align*}
|\Psi_{T+1}^{(s)}| \leq 2^s \cdot \left(5(1+\alpha^2)\sqrt{d|\Psi_{T+1}^{(s)}|} + 4\sqrt{dT} + 2\sqrt{T\ln \frac{TK}{\delta}}\right).
\end{align*}
\end{lemma}
\begin{proof}
This is true by combining \lemref{lem:s_to_sqrt}, \lemref{lem:sigma_to_sqrt}, and \lemref{lem:simple_cb} with a similar proving strategy as in Lemma 16 of~\cite{LinRel}. 
\begin{align}
\sum_{t \in\Psi_{T+1}^{(s)}} w_{t,k}^{(s)} &= (1+\alpha)\sum_{t\in\Psi_{T+1}} s_{t,k_t} + (4\sqrt{d} + 2\sqrt{\ln \frac{TK}{\delta}})\sum_{t\in\Psi_{T+1}} \|\si_{t,k}\|_{\infty} \label{eq:w_upper1} \\
&\leq 5 \cdot (1+\alpha^2) \sqrt{d|\Psi_{T+1}|} + (4\sqrt{d} + 2\sqrt{\ln \frac{TK}{\delta}})\sqrt{|\Psi_{T+1}|} \label{eq:w_upper2} \\
&\leq 5 \cdot (1+\alpha^2) \sqrt{d|\Psi_{T+1}|} + 4\sqrt{dT} + 2\sqrt{T\ln \frac{TK}{\delta}}, \label{eq:w_upper3}
\end{align}
where \eqnref{eq:w_upper2} is due to \lemref{lem:s_to_sqrt} and \lemref{lem:sigma_to_sqrt}. By Line \ref{algl:step2b} of \algref{alg:suprdn}, we have
\begin{align}
\sum_{t \in\Psi_{T+1}^{(s)}}  w_{t,k}^{(s)} \geq 2^{-s}|\Psi_{T+1}^{(s)}|. \label{eq:w_lower}
\end{align}
Combine \eqnref{eq:w_upper3} and \eqnref{eq:w_lower} yields this lemma.
\end{proof}

\begin{theorem}\label{thm:regret_bound}
If SupREN is run with $\alpha=\sqrt{\frac{1}{2}\ln \frac{2TK}{\delta}}$, with probability at least $1 - \delta$, the regret of the algorithm is
\begin{align}
O\left(\sqrt{Td\ln^3\left(\frac{KT\ln(T)}{\delta}\right)}\right).
\end{align}
\end{theorem}
\begin{proof}
The proof is an extension of Theorem 6 in \cite{LinRel} to handle the uncertainty in item embeddings. We denote as $\Psi_0$ the set of trials for which an alternative is chosen in Line \ref{algl:step2c} of \algref{alg:suprdn}. Note that $2^{-S}\leq\frac{1}{\sqrt{T}}$; therefore $\{1,\dots,T\}=\Psi_0\cup \bigcup_s \Psi^{(s)}_{T+1}$. We have
\begin{align}
E[B(T)] 
&= \sum_{t=1}^T [E[r_{t,k_t^*}] - E[r_{t,k_t}] \nonumber \\
&= \sum_{t\in \Psi_0} [E[r_{t,k_t^*}] - E[r_{t,k_t}]
+ \sum_{s=1}^S \sum_{t\in \Psi_{T+1}^{(s)}} [E[r_{t,k_t^*}] - E[r_{t,k_t}] \nonumber \\
&\leq \frac{2}{\sqrt{T}}|\Psi_0| + \sum_{s=1}^S 8 \cdot 2^{-s}\cdot |\Psi_{T+1}^{(s)}| \label{eq:regret_to_psi} \\
&\leq \frac{2}{\sqrt{T}}|\Psi_0| + \sum_{s=1}^S 8 \cdot \left(5(1+\alpha^2)\sqrt{d|\Psi_{T+1}^{(s)}|} + 4\sqrt{dT} + 2\sqrt{T\ln \frac{TK}{\delta}}\right) \label{eq:psi_to_sqrt} \\
&\leq 2\sqrt{T} + 40(1+\ln \frac{2TK}{\delta})\sqrt{STd} + 32S\sqrt{dT} + 16S\sqrt{T\ln \frac{TK}{\delta}},
\end{align}
with probability $1-2\delta S$. \eqnref{eq:regret_to_psi} is by \lemref{lem:regret_to_width}, and \eqnref{eq:psi_to_sqrt} is by \lemref{lem:psi_to_sqrt}. By the Azuma-–Hoeffding inequality (\lemref{lem:azuma}) with $B=2\sqrt{2T}\sqrt{\ln\frac{2}{\delta}}$ and $a_{\tau}=2$, we have
\begin{align}\label{eq:after_azuma}
B(T) \leq 2\sqrt{T} + 44\cdot(1 + \ln \frac{2TK}{\delta})\sqrt{STd} + 32S\sqrt{dT} + 16S\sqrt{T\ln\frac{TK}{\delta}},
\end{align}
with probability at least $1-2\delta(S+1)$. To see this, note that $1-2\delta(S+1) < 1-2\delta S - \delta$ and that
\begin{align*}
2\sqrt{2T}\sqrt{\ln\frac{2}{\delta}}
\leq 4\sqrt{T}\sqrt{\ln \frac{2TK}{\delta}} 
\leq 4\cdot(1+\ln\frac{2TK}{\delta})\sqrt{STd}.
\end{align*}
Replacing $\delta$ by $\frac{\delta}{2S+2}$ and $S$ by $\ln T$ in \eqnref{eq:after_azuma} along with simplification gives us
\begingroup\makeatletter\def\f@size{8}\check@mathfonts
\begin{align*}
B(T)
&\leq 2\sqrt{T} + 44\cdot(1+\ln \frac{2TK(2S+2)}{\delta})\sqrt{T\ln T}\sqrt{d} + 32S\sqrt{dT} + 16S\sqrt{T\ln\frac{TK(2S+2)}{\delta}}\\
&\leq 2\sqrt{T} + 44\cdot(1+\ln \frac{2TK(2S+2)}{\delta})(1+\ln T)^{\frac{1}{2}}\sqrt{Td}  + 32S\sqrt{dT} + 16\ln T\sqrt{\ln\frac{TK(2S+2)}{\delta}}\sqrt{T}\\
&\leq 2\sqrt{T} + 44\cdot(1+\ln \frac{2TK(2\ln T+2)}{\delta})^{\frac{3}{2}}\sqrt{Td}\\
&\qquad\qquad + 32\cdot(1+\ln \frac{2TK(2\ln T + 2)}{\delta})\sqrt{dT} + 16\cdot(1+\ln \frac{2TK(2\ln T + 2)}{\delta})^{\frac{3}{2}}\sqrt{Td} \\
&\leq 2\sqrt{T} + 92\cdot(1+\ln \frac{2TK(2\ln T+2)}{\delta})^{\frac{3}{2}}\sqrt{Td},
\end{align*}
\endgroup
with probability $1-\delta$. Therefore we have
\begin{align*}
B(T)\leq 2\sqrt{T} + 92\cdot(1+\ln \frac{2TK(2\ln T+2)}{\delta})^{\frac{3}{2}}\sqrt{Td}=O(\sqrt{Td\ln^3(\frac{KT\ln(T)}{\delta}}),
\end{align*}
with probability $1-\delta$. 
\end{proof}
\thmref{thm:regret_bound} shows that even with the uncertainty in the item embeddings, our proposed REN can achieve the same rate-optimal sublinear regret bound as in~\cite{chu2011contextual}.

\begin{figure}[!tb]
\begin{center}
%\framebox[4.0in]{$\;$}
%\includegraphics[height=5cm]{likeli1.eps}
\subfigure{
\includegraphics[height=4.2cm]{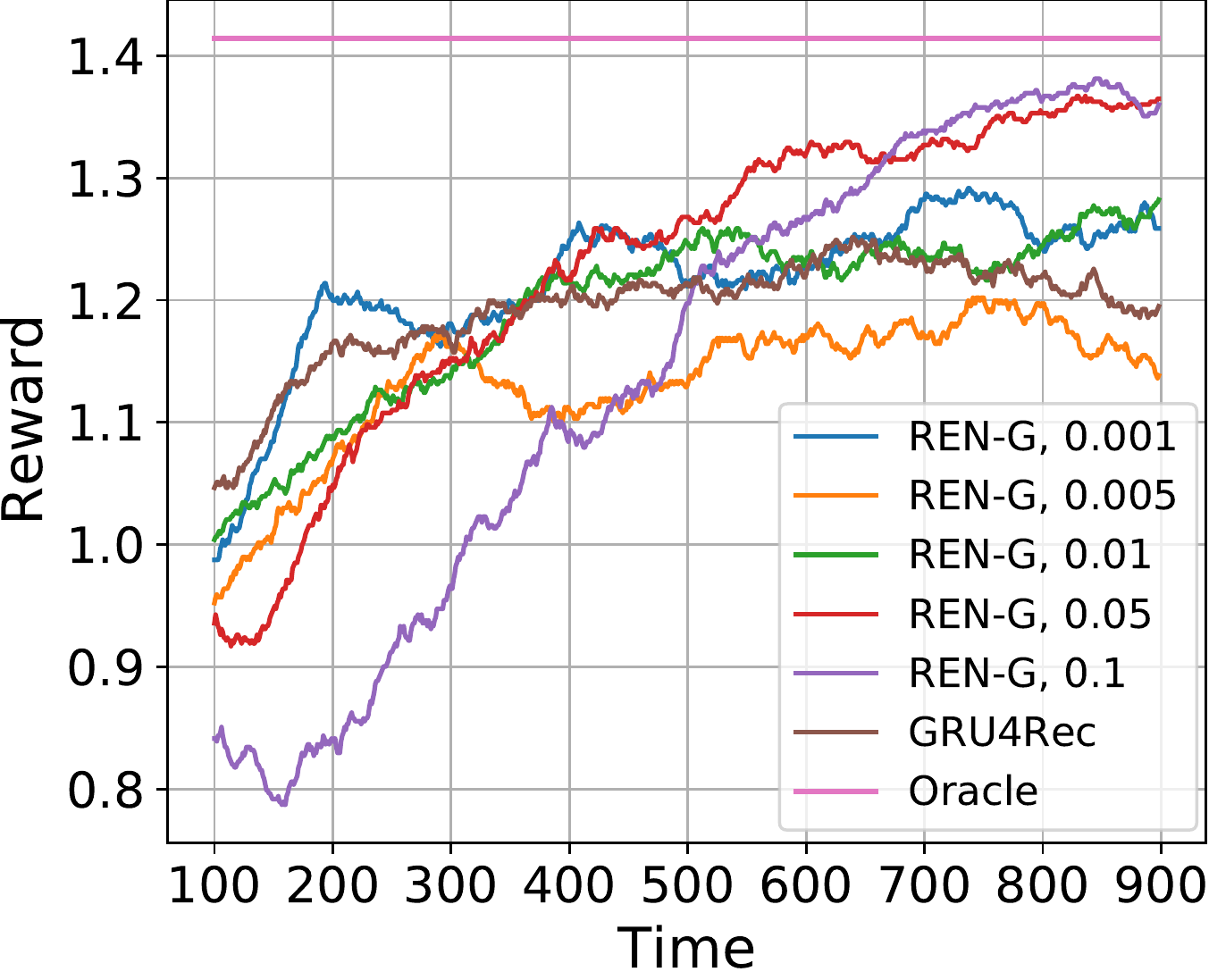}}
\subfigure{
\includegraphics[height=4.2cm]{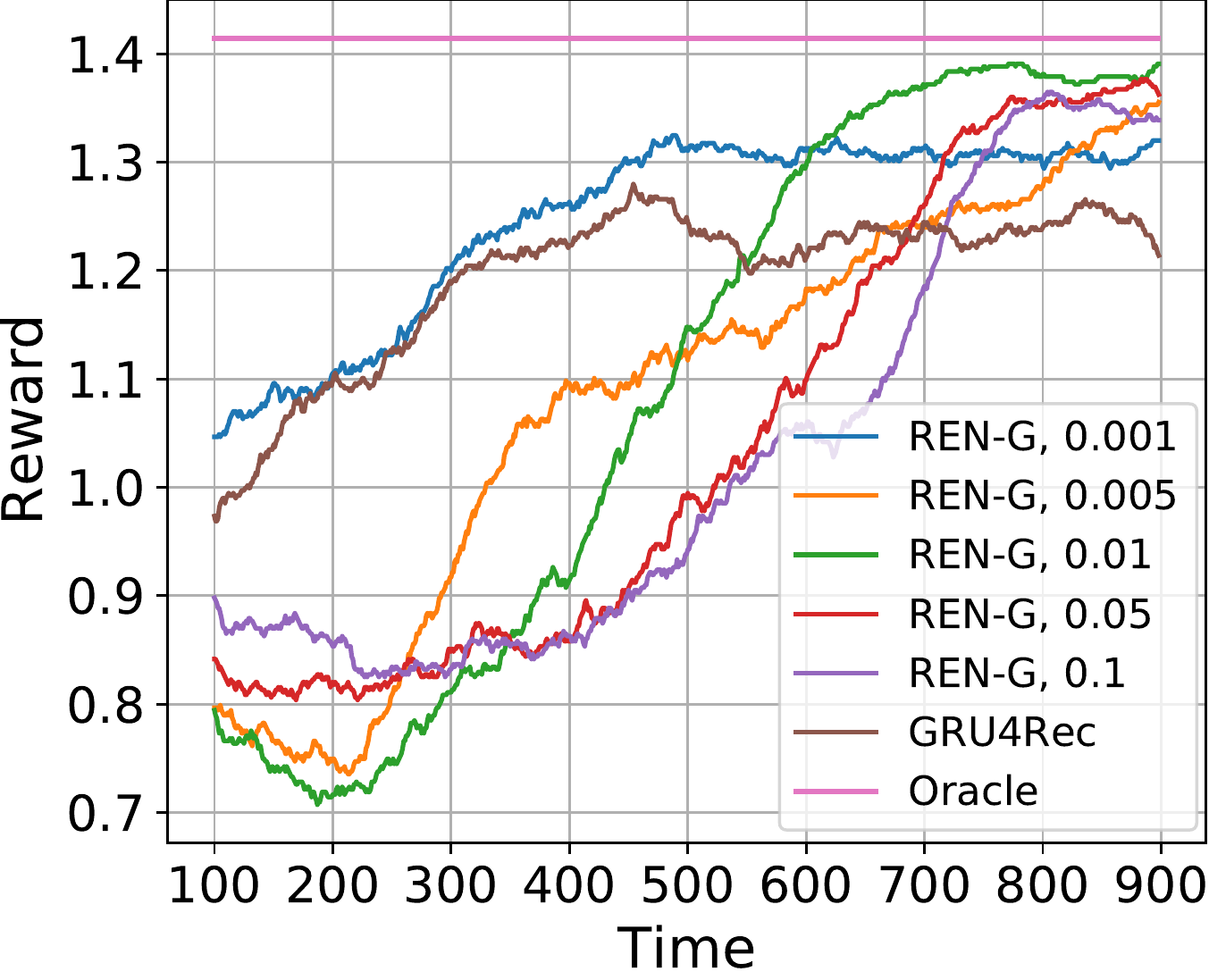}}
\subfigure{
\includegraphics[height=4.2cm]{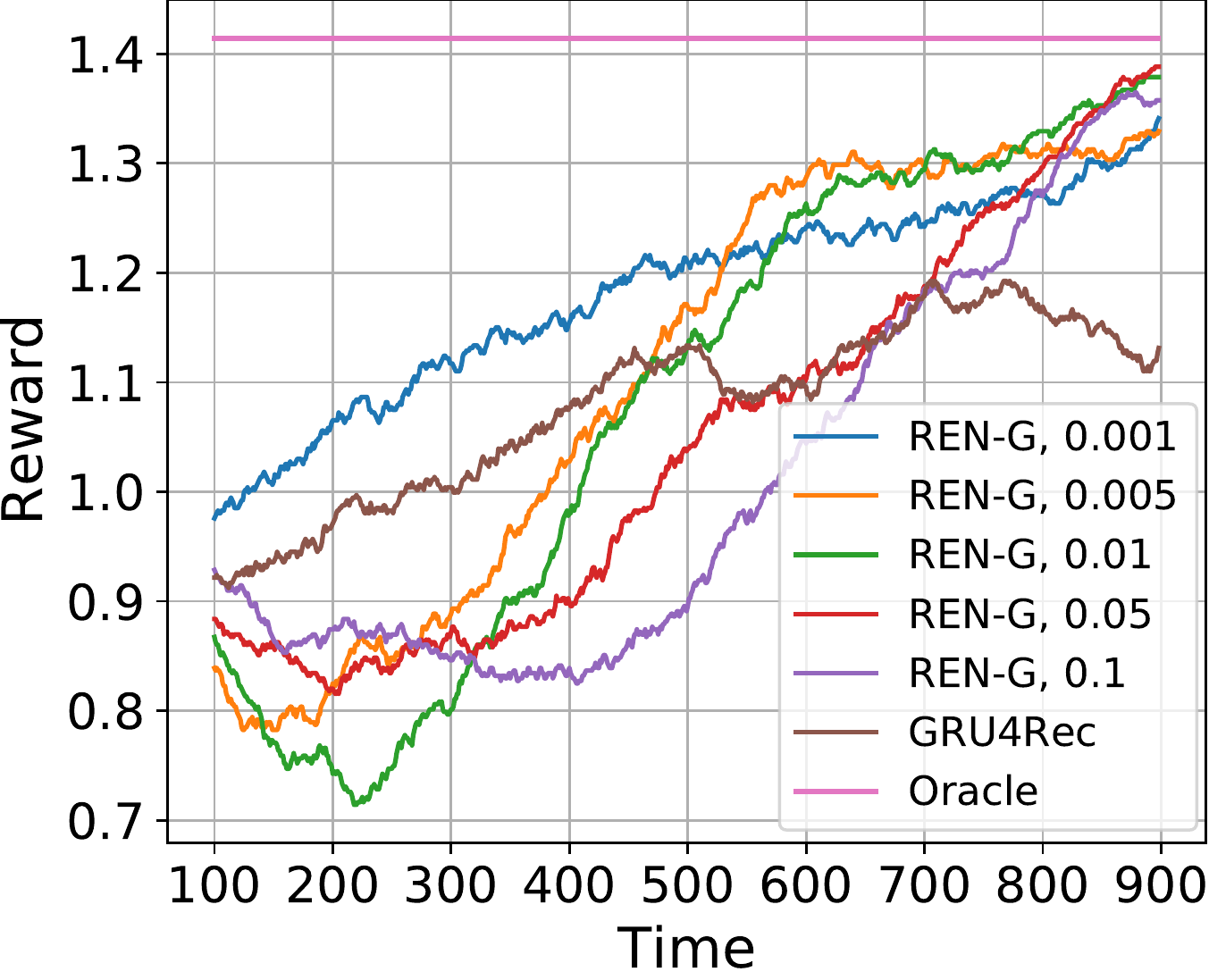}}
\end{center}
\vskip -0.2in
\caption{\label{fig:sensitivity}Hyperparameter sensitivity for $\lambda_d$ in \emph{SYN-S}, \emph{SYN-M}, and \emph{SYN-L}.
}
\vskip -0.0in
\end{figure}

\section{More Details on Datasets}
\subsection{\emph{MovieLens-1M}}
We use \emph{MovieLens-1M}~\cite{harper2016movielens} containing $3{,}900$ movies and $6{,}040$ users. Each user has $120$ interactions, and we follow the joint learning and exploration procedure described in the main paper to evaluate all methods.
% Specifically, we randomly select $1{,}000$ users from \emph{MovieLens-1M}, where each user has $120$ interactions, and follow the joint learning and exploration procedure described in the main paper to evaluate all methods.

\subsection{\emph{Trivago}}
\emph{Trivago} is a hotel recommendation dataset with $730{,}803$ users, $926{,}457$ items, and $910{,}683$ interactions. We use a subset with $57{,}778$ users, $387{,}348$ items, and $108{,}713$ interactions and slice the data into $M=48$ one-hour time intervals for the online experiment. Different from \emph{MovieLens-1M}, Triavago has impression data available. Specifically, at each time step, besides which item is clicked by the user, we also know which $25$ items are being shown to the user. Essentially the RecSys Challenge is a reranking problem with candidate sets of size $25$.
%\hw{TODO: fill in numbers, mention subset of datasets}

% $730{,}803$ users

% \begin{wrapfigure}{R}{0.31\textwidth}
% \centering
% \vskip -0.2in
% \includegraphics[width=0.31\textwidth]{figures/sensitivity-r-f50.pdf}
% \vskip -0.2cm
% % \captionsetup{font={scriptsize}}
% \caption{\label{fig:sensitivity}Hyperparameter sensitivity for $\lambda_r$.}
% \vskip -0.3cm
% \end{wrapfigure}

% \begin{figure}[!tb]
% \begin{center}
% %\framebox[4.0in]{$\;$}
% %\includegraphics[height=5cm]{likeli1.eps}
% \subfigure{
% \includegraphics[height=3.6cm]{figures/netflix-recall}}
% \subfigure{
% \includegraphics[height=3.6cm]{figures/netflix-in-recall}}
% \subfigure{
% \includegraphics[height=3.6cm]{figures/netflix-tune-recall}}
% \end{center}
% \vskip -0.2in
% \caption{\label{fig:netflix_recall}Rewards over time on \emph{Netflix}. One time step represents $100$ recommendations to a user.
% }
% \vskip -0.0in
% \end{figure}

\subsection{\emph{Netflix}}
Our main conclusion with \emph{Netflix} experiments is that REN-inference-only procedure collects more diverse data points about a user, which allows us to build a more generalizable user model, which leads to better long-term rewards.
The main paper demonstrates better generalizability by comparing precision@100 reward on a holdout item set, where the items are inaccessible to the user - i.e., we never collect feedback on these holdout items in our simulations.
Instead, recommendations are made by comparing the users' learned embeddings and the pretrained embeddings of the holdout items.

\section{Hyperparameters and Neural Network Architectures}
For the base models GRU4Rec, TCN, and HRNN, we use identical network architectures and hyperparemeters whenever possible following~\cite{GRU4Rec,TCN,HRNN}. Each RNN consists of an encoding layer, a core RNN layer, and a decoding layer. We set the number of hidden neurons to $32$ for all models including REN variants. \figref{fig:sensitivity} shows the REN-G's performance for different $\lambda_d$ (note that we fix $\lambda_u = \sqrt{10}\lambda_d$) in \emph{SYN-S}, \emph{SYN-M}, and \emph{SYN-L}. We can observe stable REN performance across a wide range of $\lambda_d$. As expected, REN-G's performance is closer to GRU4Rec when $\lambda_d$ is small.

\end{document}